\documentclass{article} 
\usepackage[usenames,dvipsnames]{color}
\usepackage{amsmath} 
\usepackage{amssymb} 
\usepackage{bm}
\usepackage{amsthm}
\usepackage{CJK}
\usepackage{indentfirst}
\usepackage{tikz}
\usepackage{amsmath,amsfonts,amsthm,mathrsfs}
\usepackage{hyperref}
\numberwithin{equation}{section}
\usepackage{geometry}

\numberwithin{equation}{section}
\newtheorem{theorem}{Theorem}
\newtheorem{lemma}{Lemma}
\newtheorem{proposition}{Proposition}
\newtheorem{corollary}{Corollary}
\newtheorem{definition}{Definition}
\newtheorem{remark}{Remark}
\newtheorem{example}{Example}

\usepackage{algorithm,algpseudocode}
\usepackage{amsthm}

\def\NN{\mathbb{N}}
\def\RR{\mathbb{R}}
\def\BB{\mathbb{B}}
\def\SS{\mathbb{S}}
\def\ZZ{\mathbb{Z}}

\def\TT{\mathbb{T}}
\def\bw{{\bm{w}}}
\def\bx{{\bm{x}}}
\def\by{{\bm{y}}}

\def\bb{{\bm{b}}}
\def\Nb{N_{\bm{b}}}
\def\bk{{\bm{k}}}
\def\bl{{\bm{l}}}

\def\mn{\mathcal{N}}
\def\mal{\max\limits}
\def\sul{\sum\limits}

\def\lt{\left}
\def\rt{\right}
\newcommand{\lrt}[1]{\left({#1}\right)}

\title{Do Neural Networks Really Beat the Curse of Dimensionality?\\ A Bit-Complexity View}
\date{}

\author{Tong Mao
\thanks{King Abdullah University of Science and Technology, Thuwal 23955, Saudi Arabia \\ Email: tongmao1996@gmail.com}
\and
 Jinchao Xu
\thanks{King Abdullah University of Science and Technology, Thuwal 23955, Saudi Arabia \\ Email: jinchao.xu@kaust.edu.sa}}

\begin{document}

\maketitle

\begin{abstract}
    
Traditional approximation theory measures convergence rates in terms of the number of parameters or degrees of freedom. However, practical computation operates under finite precision: parameters must be encoded using a finite number of bits. Therefore, approximation efficiency should be evaluated in terms of computational bit complexity, which is intrinsically connected to the metric entropy of the underlying function class.

In this work, we develop a unified approximation framework based on binary encoding and metric entropy. We analyze classical methods (including polynomial approximation, sparse grids, and finite elements) as well as shallow and deep neural networks, and compare their approximation rates for function classes with comparable metric entropy. We observe that, when evaluated in terms of bits, most classical methods are in general suboptimal relative to the intrinsic limits dictated by metric entropy, while neural network methods may exhibit different behaviors. We show that when complexity is measured in bits rather than parameters, no method fundamentally exceeds the approximation order achieved by classical approaches.

Our results also indicate that many seeming advantages of neural networks, including dimension-independent rates and superconvergence phenomena, stem from differences in function class complexity rather than intrinsic architectural superiority. In this sense, the traditional curse of dimensionality can be misleading; the fundamental limitation is instead a curse of bit complexity, governed by metric entropy.
\end{abstract}

	\section{Introduction}\label{sec:intro}
High-dimensional approximation problems arise naturally in many areas of scientific computing and machine learning. In such settings, one of the most familiar obstacles is the \emph{curse of dimensionality}: the number of degrees of freedom required to achieve a prescribed accuracy often grows rapidly with the dimension. This phenomenon appears throughout classical approximation theory and numerical analysis, with Sobolev approximation providing a prototypical example, where the convergence rate explicitly depends on the ratio between smoothness and dimension.

To make this dependence precise, we recall a standard framework in approximation theory. For a function class $\mathcal{F} \subset X$, one considers a sequence of approximation sets $\{\mathcal{A}_N\}_{N\in\mathbb{N}}$, each having $\mathcal{O}(N)$ degrees of freedom, and studies how accurately elements of $\mathcal{F}$ can be represented by $\mathcal{A}_N$. More precisely, one investigates the worst-case error over the unit ball $\BB(\mathcal{F})$ of $\mathcal{F}$:
\begin{equation}\label{eqn_rate_gener}
    \sup_{f\in\BB(\mathcal{F})} \inf_{f_N \in \mathcal{A}_N} \|f - f_N\|_X\simeq N^{-\alpha},
\end{equation}
where $\simeq$ indicates equivalence up to constants (see \eqref{eqn_def_lesssim} below). The rate $\mathcal{O}(N^{-\alpha})$ with respect to $N$ is referred to as the \emph{approximation rate} of the method for the class $\mathcal{F}$.

A prototypical example arises when approximating functions in the Sobolev space $\mathcal{H}^r([0,1]^d)$. Let $\Pi_n$ denote the space of polynomials of degree less than $n$ in each variable, whose dimension satisfies $\dim(\Pi_n)\simeq n^d$. Then the optimal $L^2$-approximation rate is given by (see, e.g., \cite{devore1993constructive,devore1989optimal})
\begin{equation}\label{eqn_sharp_poly}
\sup_{f\in\BB(\mathcal{H}^{r}([0,1]^d))} \inf_{f_N \in \Pi_n} \|f - f_N\|_{L^2([0,1]^d)} \simeq n^{-r}
\simeq N^{-\frac{r}{d}},
\end{equation}
where $N:=\dim(\Pi_n)\simeq n^d$.

The same rate holds for piecewise polynomial approximation. For instance, let $V_N$ be a finite element space associated with a quasi-uniform mesh on $[0,1]^d$ and having $N$ degrees of freedom. Then one has \cite{lin2014lower,ciarlet2002finite}
\begin{equation}\label{eqn_sharp_fe}
\sup_{f\in\BB(\mathcal{H}^{r}([0,1]^d))}
\inf_{f_N \in V_N}
\|f - f_N\|_{L^2([0,1]^d)}
\simeq N^{-\frac{r}{d}}.
\end{equation}
The dependence of the exponent on the ratio $\frac{r}{d}$ reflects the classical \emph{curse of dimensionality}: for fixed smoothness $r$, the convergence rate deteriorates as the dimension $d$ increases.

The dimension-independent approximation rate
$$\|f-f_N\|_{L^\infty([0,1]^d)}\lesssim N^{-\frac{1}{2}}$$
for shallow neural networks with $\mathcal{O}(N)$ parameters has been extensively studied in the literature (e.g. \cite{mhaskar1994dimension,devore1996some,klusowski2018approximation,montanelli2019new,siegel2020approximation,ma2022uniform,mao2022approximation}). In particular, for the ReLU$^k$-type Barron space (see Definition \ref{def:variation} below), recent study \cite{siegel2022high,siegel2022sharp,siegel2023optimalzonoids} showed the shallow ReLU$^k$ networks with $\mathcal{O}(N)$ parameters achieve the sharp rate
\begin{equation}
    \sup\limits_{f\in\BB(\mathcal{B}^k([0,1]^d))}\|f-f_N\|_{L^\infty([0,1]^d)}\lesssim N^{-\frac{1}{2}-\frac{2k+1}{2d}}.
\end{equation}
Such rates appear to surpass the classical Sobolev-type rate $\mathcal{O}(N^{-r/d})$ and are often interpreted as evidence that neural networks overcome the curse of dimensionality.

Interestingly, recent works \cite{shen2022optimal,yang2023optimal,yang2023nearly,siegel2023optimal} have shown that deep ReLU networks can achieve approximation rates in Sobolev spaces that exceed the classical optimal rate $\mathcal{O}(N^{-\frac{r}{d}})$ in \eqref{eqn_sharp_poly} and \eqref{eqn_sharp_fe}. Specifically, there exists a deep ReLU network $f_N$ of width $25d+31$ and depth $\mathcal{O}(N)$ such that
\begin{equation}
    \|f-f_N\|_{L^2([0,1]^d)}\lesssim\|f\|_{\mathcal{H}^{r}([0,1]^d)}N^{-\frac{2r}{d}}.
\end{equation}

Beyond approximation rates, another line of research studies whether a neural network with fixed number of parameters can approximate every continuous function arbitrarily well, i.e., whether the corresponding function class is dense in $C([0,1]^d)$. A series of works \cite{maiorov1999lower,yarotsky2020phase,zhang2022deep,shen2020deep,jiao2021deep,shen2021neural} developed and strengthened this universality theory, revealing increasingly extreme forms of approximation capability. The following example, simplified from \cite{yarotsky2020phase,zhang2022deep}, illustrates a particularly striking manifestation of this phenomenon:
\begin{equation}\label{eqn:3parameter_appr}
\inf_{f_3\in\Sigma_3}\max_{x\in[0,1]}
\left|f(x)-f_3(x)\right|=0,\quad f\in C([0,1]),
\end{equation}
where
$$\Sigma_3=\{\theta_1\cos\!\left(\theta_2\cos\!\left(\lfloor\theta_3\circ\rfloor\right)\right):~\theta_1,\theta_2,\theta_3\in\RR\}$$
and $\lfloor \cdot \rfloor$ denotes the floor function (the greatest integer less than or equal to its argument).

This shows that the approximation error of the neural network architecture $\Sigma_3$ can be made arbitrarily small while the number of parameters remains fixed at $N=3$. Interpreted in terms of parameter-based approximation rates \eqref{eqn_rate_gener}, this would formally amount to
\[
\inf_{f_3\in\Sigma_3}\|f-f_3\|_{L^\infty([0,1])}=\mathcal O(3^{-\infty}),
\]
indicating that parameter count alone is not a meaningful measure of approximation efficiency.

Moreover, combined with the Kolmogorov–Arnold representation theorem \cite{kolmogoro1956representation,arnold1958representation}, which expresses multivariate continuous functions as superpositions of univariate ones, a neural network with only fixed number of parameters can approximate any multivariate continuous function to arbitrary accuracy. Such results suggest that high dimensionality itself may not be the true source of difficulty. These observations raise a fundamental question: is the so-called curse of dimensionality intrinsic to approximation, or does it arise from the way approximation complexity is measured? This question motivates our investigation from the viewpoint of metric entropy and bit-complexity.

The concept of metric entropy was first introduced by Kolmogorov and Uspenskii \cite{kolmogorov1958definition}. For a compact subset $\mathcal{K}$ of a Banach space, its metric entropy is the infimum of radii $\varepsilon$ such that $\mathcal{K}$ can be covered by $2^m$ balls of radius $\varepsilon$. This quantity characterizes the complexity of $\mathcal{K}$ and directly determines the bounds of its covering number $\mn(\mathcal{K},\varepsilon)_{\mathcal X}$ (see Proposition \ref{prop:entropy_and_covering}). 

Besides metric entropy, other concepts such as the Vapnik--Chervonenkis (VC) dimension \cite{vapnik2015uniform,bartlett1996sample,harvey2017nearly} and the Rademacher complexity \cite{bartlett2002rademacher,mohri2018foundations,koltchinskii2001rademacher} are also widely used to quantify the complexity of function classes. It has long been recognized that the complexity of the hypothesis class, such as that induced by a neural network, plays a central role in generalization analysis \cite{barron1991complexity,cucker2002mathematical,chen2004support,feng2021generalization}, and that generalization error depends on both approximation error and hypothesis-class complexity \cite{cucker2007learning}. In the present work, metric entropy is particularly convenient because it is closely connected to covering numbers, admits a direct interpretation in terms of finite-bit encoding, and has a well-established role in classical approximation theory (see also \cite{kulkarni1989metric}).

Therefore, it is natural to connect metric entropy with approximation theory. In this paper, we study how metric entropy governs approximation rates in machine learning. We show that it provides a unified explanation of the phenomena discussed above: the rate of approximation is fundamentally determined not by the dimension, but by the metric entropies of the target function class and the neural network, which together yield intrinsic lower bounds on approximation accuracy.

This paper is organized as follows. In Section \ref{sec:metric_entropy}, we introduce the notion of metric entropy and its connection to bit-complexity through digital encoding. Section \ref{sec:classical} studies classical approximation methods, including polynomial approximation, finite elements, and sparse grids, from the viewpoint of both parameters and binary bits. Section \ref{sec:appr_nn} analyzes shallow and deep neural network approximation under the same framework and highlights the differences between parameter-based and bit-based rates. In Section \ref{sec:entropy_appr}, we derive approximation lower bounds using metric entropy and explain the observed suboptimality of various methods. Finally, Section \ref{sec:conclusion} summarizes the main conclusions and discusses several implications, including dimension-independent rates, superconvergence phenomena, and practical considerations related to finite-precision representations.

Throughout the paper, we use the following notation.
\begin{enumerate}
    \item Following \cite{xu1992iterative}, we adapt the notation $\gtrsim$, $\lesssim$, and $\simeq$, which express upper and lower bounds up to constant factors.  When we write
\begin{equation}\label{eqn_def_lesssim}
    f(x)\gtrsim g(x),\quad g(x)\lesssim h(x),\quad h(x)\simeq k(x),
\end{equation}
it means that there exist positive constants $c,C$ independent of $x$ such that
$$
f(x)\ge c g(x),\quad g(x)\le C h(x),\quad c k(x)\le h(x)\le C k(x).
$$
\item Given a Banach space $\mathcal{X}$, the unit ball in $\mathcal{X}$ is denoted as
\begin{equation}\label{eqn:def_unit_ball}
	\BB\lrt{\mathcal{X}}:=\lt\{v\in\mathcal{X}:~\|v\|_{\mathcal{X}}\leq1\rt\}.
\end{equation}
\end{enumerate}

\section{Metric entropy and digital encoding}\label{sec:metric_entropy}

In practical implementations, approximation methods operate under finite-precision arithmetic: parameters are stored and manipulated with limited accuracy. To connect encoding with approximability, we recall the notion of \emph{metric entropy}, introduced by Kolmogorov and Uspenskii \cite{kolmogorov1958definition}, which quantifies how many bits are needed to represent elements of a compact set up to a given accuracy.

\subsection{Digital encoding}
Encodability quantifies how efficiently elements of a set can be represented using a fixed number of bits. A real number $y\in[-A,A]$ can be encoded in binary as
\[
y=\mathrm{sgn}(y)\sum_{k=-\infty}^{\lfloor \log_2 A \rfloor} y_k\,2^k,\qquad y_k\in\{0,1\}.
\]
Using $\tau$ bits, a real number is approximated as
\begin{equation}\label{eqn:bit_repre}
    y^{[\tau]}=\mathrm{sgn}(y)\sum_{k=\lfloor \log_2 A \rfloor-\tau+2}^{\lfloor \log_2 A \rfloor} y_k\,2^k,\qquad y\in[-A,A],
\end{equation}
yielding
\begin{equation}\label{eqn:bin_pos}
    |y-y^{[\tau]}|\le A\,2^{-\tau+2}.
\end{equation}
Denote
\[
\{0,1\}^\tau:=\{(e_1,\dots,e_\tau):~e_i\in\{0,1\}\},
\]
then the map $y\mapsto y^{[\tau]}$ in \eqref{eqn:bin_pos} induces $[-A,A]\to\{0,1\}^{\tau}$ via
\[
y\mapsto\Big(\tfrac{\mathrm{sgn}(y)+1}{2},~(y_k)_{k=\lfloor \log_2 A \rfloor-\tau+2}^{\lfloor \log_2 A \rfloor}\Big),
\]
so the image has cardinality $2^\tau$, determined by the bit budget $\tau$.

We extend the digital encoding from intervals to compact sets. An encoding of $\mathcal{K}$ maps each element of $\mathcal{K}$ to a binary word $(e_1,\dots,e_m)\in\{0,1\}^m$.
\begin{definition}
Let $m\in\NN$. An encoding of a compact set $\mathcal{K}\subset\mathcal{X}$ with $m$ bits is a map
\[
F:\mathcal{K}\to \{\phi_1,\dots,\phi_{2^m}\},\qquad \phi_j\in\mathcal{X},
\]
with encoding accuracy
\[
\sup_{\phi\in\mathcal{K}}\|\phi-F(\phi)\|_{\mathcal{X}}.
\]
\end{definition}

\begin{example}[Encoding the cube]\label{ex:cube}
\rm For $\bm a=(a_1,\dots,a_d)\in[-A,A]^d$, using $\tau$ bits per component yields an $m=d\tau$-bit encoding
\[
(a_1,\dots,a_d)\mapsto (a_1^{[\tau]},\dots,a_d^{[\tau]}).
\]
For $A=1$,
\[
\|\bm a^{[\tau]}-\bm a\|_\infty=\max_{1\le j\le d}|a_j^{[\tau]}-a_j|\le 2^{-\tau+2}\le 2^{-\frac{m}{d}+2}.
\]
\end{example}

Often $\mathcal{X}$ is a function space. Then $f\in\mathcal{K}$ is approximated by parameters $(a_1,\dots,a_N)$, which are encoded componentwise as above.

\begin{definition}
A parameterized encoding of $\mathcal{K}$ is a composition
\[
G:\ f\mapsto \bm a(f)=(a_1(f),\dots,a_n(f)),\qquad
H:\ \bm a(f)^{[\tau]}\mapsto T_n(f)^{[\tau]},
\]
where $\bm a(f)\in[-A,A]^n$, and $\bm a(f)\mapsto \bm a(f)^{[\tau]}$ is as in Example~\ref{ex:cube}.
\end{definition}

\subsection{Definition and simple examples}
Given a compact set $\mathcal{K}$ and a bit budget $m$, the minimal achievable uniform error with $m$ bits is precisely the metric entropy.

\begin{definition}[Metric Entropy]
Let $\mathcal{X}$ be a Banach space and $\mathcal{K}\subset\mathcal{X}$. The (dyadic) entropy number is
\begin{equation}\label{eqn:def_metric_entropy}
\epsilon_m(\mathcal{K})_{\mathcal{X}}
=\inf\left\{\epsilon>0:\ \mathcal{K}\ \text{can be covered by }2^m\text{ balls of radius }\epsilon\text{ in }\mathcal{X}\right\}.
\end{equation}
\end{definition}

\begin{example}
\rm \begin{enumerate}
\item For $[0,1]^d$ with $\ell^\infty$-norm, let $n\in\NN$ satisfy $n^d\le 2^m<(n+1)^d$. Partition into $n^d$ cubes of side $\frac{1}{n}$. Then
\[
\frac{1}{2(2^{\frac{m}{d}}+2)}\le \frac{1}{2(n+1)}\le \epsilon_m([0,1]^d)_{\ell^\infty}\le \frac{1}{n}\le \frac{1}{2^{\frac{m}{d}}-1},
\]
where $n=\lfloor 2^{\frac{m}{d}}\rfloor$.
\item For $[0,M]\subset\RR$,
\[
\frac{M}{2^{m+1}} \le \epsilon_m([0,M]) \le \frac{M}{2^m}.
\]
When $M>2^{m+1}$, $\epsilon_m([0,M])$ exceeds that of $[0,1]^d$, showing that lower dimension need not imply smaller complexity.
\end{enumerate}
\end{example}
To relate metric entropy to more classical notions of set complexity, we introduce the covering number.

\begin{definition}[Covering Number]
For $\mathcal{K}\subset\mathcal{X}$ and $\epsilon>0$,
\begin{equation}\label{eqn:def_covering}
\mathcal{N}(\mathcal{K},\epsilon)_{\mathcal{X}}
:=\inf\{N\in\NN:\ \mathcal{K}\ \text{can be covered by }N\text{ balls of radius }\epsilon\text{ in }\mathcal{X}\}.
\end{equation}
\end{definition}
Metric entropy and covering numbers are equivalent up to a dyadic transformation, and can be related through the following inequalities.
\begin{proposition}[Metric Entropy and Covering Numbers]\label{prop:entropy_and_covering}
\it If $\epsilon_1< \epsilon_m(\mathcal{K})_{\mathcal{X}}< \epsilon_2$, then
\begin{equation}\label{eqn:entropy_and_covering}
\log_2 \mathcal{N}(\mathcal{K},\epsilon_2)_{\mathcal{X}} \le m \le \log_2 \mathcal{N}(\mathcal{K},\epsilon_1)_{\mathcal{X}}.
\end{equation}
\end{proposition}

\subsection{Metric entropy and bit-complexity}

Metric entropy provides a fundamental, scheme-independent characterization of \emph{bit-complexity}. More precisely, the dyadic entropy number $\epsilon_m(\mathcal{K})_{\mathcal{X}}$ represents the minimal worst-case approximation error achievable when elements of $\mathcal{K}$ are represented using $m$ binary bits. In this sense, metric entropy is not merely related to bit-complexity — it quantitatively \emph{is} bit-complexity.

\begin{proposition}[Metric Entropy and Bit-Complexity]\label{prop:entropy_equiv}
\it Let $\mathcal{K}\subset\mathcal{X}$ and let $\mathcal{M}$ be the set of maps $\phi:\{0,1\}^m\to\mathcal{X}$. Then
\[
\inf_{\phi\in\mathcal{M}}\ \sup_{f\in\mathcal{K}}\ \inf_{y\in\{0,1\}^m}\ \|f-\phi(y)\|_{\mathcal{X}}
=\epsilon_m(\mathcal{K})_{\mathcal{X}}.
\]
\end{proposition}

\begin{proof} 
For any $\eta > 0$, consider $2^m$ balls of radius $\epsilon_m(\mathcal{K})_{\mathcal{X}} + \eta$ covering $\mathcal{K}$. Let $\{g_k\}_{k=1}^{2^m} \subset \mathcal{X}$ be the centers of these balls, and define the mapping
\[
\phi_\eta: \{0,1\}^m \to \{g_k\}_{k=1}^{2^m}.
\]
This map satisfies
\[
\sup_{f \in \mathcal{K}} \inf_{y \in \{0,1\}^m} \|f - \phi_\eta(y)\|_{\mathcal{X}} = \epsilon_m(\mathcal{K})_{\mathcal{X}} + \eta.
\]
Letting $\eta \to 0$ gives
\[
\inf_{\phi \in \mathcal{M}} \sup_{f \in \mathcal{K}} \inf_{y \in \{0,1\}^m} \|f - \phi(y)\|_{\mathcal{X}} \leq \epsilon_m(\mathcal{K})_{\mathcal{X}}.
\]

Conversely, for any $\phi: \{0,1\}^m \to \mathcal{X}$, the balls of radius
\[
\sup_{f \in \mathcal{K}} \inf_{y \in \{0,1\}^m} \|f - \phi(y)\|_{\mathcal{X}}
\]
centered at $\{\phi(y)\}_{y \in \{0,1\}^m}$ must cover $\mathcal{K}$. By definition of $\epsilon_m(\mathcal{K})_{\mathcal{X}}$, this implies
\[
\sup_{f \in \mathcal{K}} \inf_{y \in \{0,1\}^m} \|f - \phi(y)\|_{\mathcal{X}} \geq \epsilon_m(\mathcal{K})_{\mathcal{X}}.
\]
\end{proof}

\begin{remark}
\rm Proposition~\ref{prop:entropy_equiv} shows that $\epsilon_m(\mathcal{K})_{\mathcal{X}}$ characterizes the optimal approximation accuracy achievable under a fixed $m$-bit budget, independently of any particular representation scheme. In this sense, metric entropy captures the intrinsic information content of the function class $\mathcal{K}$. Any approximation method that claims improved performance must ultimately respect this entropy constraint. 
\end{remark}

\subsection{Metric entropies of some function classes}\label{subsec:classes_entropy}

To illustrate how metric entropy characterizes the complexity of function classes, we now present its estimates for several fundamental spaces used in approximation theory.

\begin{definition}[Sobolev spaces]\label{def:Sobolev}
Given $r,d\in\NN$ and $\Omega\subset\RR^d$ with nonempty interior,
\[
\mathcal{H}^{r}(\Omega):=\{f\in L^2(\Omega):~\|f\|_{\mathcal{H}^{r}(\Omega)}<\infty\},
\]
where
\begin{equation}\label{eqn:poly_upper_lower}
\|f\|_{\mathcal{H}^{r}(\Omega)}:=
\Big(\|f\|_{L^2(\Omega)}^2+\sul_{\bm s\in\NN^d,\ 1\leq\|\bm s\|_1\le r}\|D^{\bm s}f\|_{L^2(\Omega)}^2\Big)^{\frac{1}{2}},
\end{equation}
\end{definition}

The following sharp entropy estimate for Sobolev spaces can be found in \cite{lorentz1996constructive}).

\begin{theorem}[Birman-Solomyak]\label{thm:entropy_sob}
\it Let $\Omega \subset \mathbb{R}^d$ be a bounded domain with Lipschitz boundary,
\begin{equation}
    \epsilon_m(\BB\lrt{\mathcal{H}^{r}(\Omega)})_{L^2(\Omega)} \simeq  m^{-\frac{r}{d}}, \quad m \in \mathbb{N}.
\end{equation}
\end{theorem}

\begin{definition}[Barron space]\label{def:variation}
For $k\in\mathbb{N}$, let the ReLU$^k$ activation function be
\[
\sigma_k(t):=\max\{t,0\}^k,\qquad t\in\RR.
\]
More generally, given an activation $\sigma:\RR\to\RR$ and a parameter set $G\subset\RR^{d+1}$, define the dictionary
\[
\mathbb{D}_\sigma:=\left\{\pm\sigma(\bw\cdot\bx+b):~\binom{\bw}{b}\in G\right\},
\]
and its $L^2$-closed convex hull
\[
B_1=\overline{\mathrm{conv}(\mathbb{D}_\sigma)}
=\overline{\bigcup_{n=1}^\infty
\left\{\sum_{j=1}^n a_jh_j:~\sum_{j=1}^n a_j\leq1,\ a_j\geq0,\ h_j\in\mathbb{D}_\sigma\right\}}.
\]
Following \cite{siegel2022sharp} (see also \cite{Jones1992,devore1996some,kurkova1,kurkova2,lewicki2004approximation,temlyakov2008greedy,Barron2008}), define
\[
\mathcal{B}^\sigma(\Omega):=
\left\{f\in L^2(\Omega):~\|f\|_{\mathcal{B}^\sigma(\Omega)}<\infty\right\},
\qquad
\|f\|_{\mathcal{B}^\sigma(\Omega)}
=
\inf\left\{t>0:~f\in tB_1\right\}.
\]
For the homogeneous activation $\sigma_k$, we take $G=\SS^d$ and write
\[
\mathcal{B}^k(\Omega):=\mathcal{B}^{\sigma_k}(\Omega).
\]
\end{definition}
These spaces admit an integral representation \cite{siegel2023characterization}
\[
f(\bx)=\int_G \sigma(\bw\cdot\bx+b)\,d\mu\binom{\bw}{b},\qquad \bx\in\Omega,
\]
with a signed Borel measure $\mu$ on $G$, leading to standard shallow-network approximation results. A space admits a similar integral representation is the spectral Barron space (see, e.g., \cite{klusowski2018approximation,xu2020finite}).

\begin{definition}[Spectral Barron space]\label{def:spectralbarron}
Let $s\geq0$, the spectral Barron space is defined by
\begin{equation*}
    \mathcal{V}^s(\Omega):=\{f\in L^2(\Omega):\ \|f\|_{\mathcal{V}^s(\Omega)}<\infty\},
\end{equation*}
where
\begin{equation}    \|f\|_{\mathcal{V}^s(\Omega)}:=\inf\limits_{\substack{f_e\in L^2(\RR^d)\\f_e|_\Omega=f}}\int_{\RR^d}|\hat{f_e}(\omega)|(1+|\omega|^2)^{s\over2}d\omega
\end{equation}

\end{definition}

\begin{theorem}[\cite{siegel2022sharp}]\label{thm:entropy_variation}
\it Let $\Omega\subset\RR^d$ be bounded, then
\begin{equation}\label{eqn:entropy_variation}
\begin{split}
    &m^{-\frac{1}{2}-\frac{2k+1}{2d}}\ \lesssim\ \epsilon_m\big(\BB(\mathcal{B}^k(\Omega))\big)_{L^2(\Omega)}\ \lesssim\ m^{-\frac{1}{2}-\frac{2k+1}{2d}},
\\
&m^{-\frac{1}{2}-\frac{s}{d}}\ \lesssim\ \epsilon_m\big(\BB(\mathcal{V}^s(\Omega))\big)_{L^2(\Omega)}\ \lesssim m^{-\frac{1}{2}-\frac{s}{d}}(\log m)^{\frac{s}{d}}.
\end{split}
\end{equation}
\end{theorem}

In the next section, we analyze approximation in terms of binary bits rather than parameters. This perspective shows that the seeming advantages (e.g., superconvergence of deep networks in Sobolev spaces) can be offset by the substantially higher metric entropy of $\Sigma_{W,n}(M)$, explaining why more bits are needed to achieve the same accuracy as simpler architectures and reconciling classical approximation with neural network results under computational resource constraints.

\section{Classical approximation in terms of number of parameters and bits}\label{sec:classical}
We begin by reviewing classical approximation rates in terms of the number of parameters, where the curse of dimensionality commonly appears.

To study encodability, we examine how approximants, such as polynomials, piecewise polynomials, and neural networks, can be represented using binary bits. In practical computation, parameters cannot be stored as exact real numbers and must be encoded with finite precision. As a result, parameter-based approximations must ultimately be converted into finite-bit representations. While metric entropy characterizes the optimal approximation error under a fixed bit budget, practical methods typically rely on parameter representations as an intermediate step.

Throughout this paper, we use $N$ to denote the number of parameters and $\Nb$ to denote the number of binary bits.

A typical parameter-based approximation consists of an encoding map $S_N$ and a reconstruction map $\tilde S_N$:
\begin{equation}\label{eqn_def_SN}
    S_N:\mathcal{K}\to \RR^N,\quad f\mapsto \Theta_f,\qquad 
    \tilde S_N:\RR^N\to\mathcal{X},\quad \Theta \mapsto \tilde S_N(\Theta).
\end{equation}
For example, a periodic function in $L^2(\TT)$ has a Fourier expansion
$$f(x)=\sum\limits_{k\in\ZZ}\widehat{f}(k)e^{ikx},$$
which naturally induces encoding  and reconstruction maps
\begin{equation}
    S_{2N+1}:f\mapsto\big\{\widehat{f}(k)\big\}_{k=-N}^{N},\qquad 
    \tilde S_{2N+1}:\big\{\widehat{f}(k)\big\}_{k=-N}^{N} \mapsto \sul_{k=-N}^{N}\widehat{f}(k)e^{ikx}.
\end{equation}

The composition $\tilde S_N \circ S_N$ produces an approximant of $f$ determined by $N$ parameters, and $\|\tilde S_N \circ S_N(f)-f\|_{\mathcal{X}}$ measures the corresponding approximation error.

To obtain a finite-bit representation, each parameter is quantized using $\tau$ binary bits as in \eqref{eqn:bit_repre}, producing approximations $\{\theta_j^{[\tau]}(f)\}_{j=1}^N$. The resulting operator $\tilde S_N^{[\tau]}$ given by
$$\tilde S_N^{[\tau]}(\Theta_f):=\tilde S_N(\Theta_f^{[\tau]})$$
can therefore encode $f$ using $N_b=N\tau$ binary bits.

The total approximation error now consists of two components: the original parameter approximation error and the quantization error introduced by finite precision. If we denote the operator norm of $\tilde S_N$ by
\begin{equation*}
\mathrm{Lip}(M,\tilde S_N)_{\mathcal{X}}
=\sup_{\Theta,\Theta'\in[-M,M]^N}
\frac{\|\tilde S_N(\Theta)-\tilde S_N(\Theta')\|_{\mathcal{X}}}
{\|\Theta-\Theta'\|_\infty},
\end{equation*}
then the quantization error satisfies
\begin{equation*}
\|\tilde S_N^{[\tau]}(\Theta_f)-\tilde S_N(\Theta_f)\|_{\mathcal{X}}
\le
\mathrm{Lip}(M,\tilde S_N)_{\mathcal{X}}
\|\Theta_f^{[\tau]}-\Theta_f\|_\infty .
\end{equation*}

Given an upper bound for $|\Theta_f|_\infty$, the binary representation estimate \eqref{eqn:bin_pos} provides a bound for $|\Theta_f^{[\tau]}-\Theta_f|_\infty$. Combining this with the operator bound of $\tilde S_N$ and the triangle inequality for 
$\|\tilde S_N(\Theta_f^{[\tau]})-f\|_{\mathcal X}$ allows one to convert a parameter-based approximation into a bit-based approximation as follows.
\begin{proposition}\label{prop_para_to_bits}
\it Let $\mathcal{B},\mathcal{X},S_N,\tilde S_N$ be as above. Suppose
\[
\|S_N(f)\|_\infty\le M,
\qquad
\|\tilde S_N\circ S_N(f)-f\|_{\mathcal{X}}\lesssim N^{-\alpha},
\qquad
f\in\mathcal{B}.
\]
Then with
\begin{equation}
\tau
=
2+\big\lceil
\log_2\big(M\mathrm{Lip}(M,\tilde S_N)_{\mathcal{X}}N^\alpha\big)
\big\rceil,
\end{equation}
there exists
$\Theta_f^{[\tau]}\in[-M,M]^N$ represented by $\Nb=N\tau$ bits such that
\begin{equation}
\|\tilde S_N(\Theta_f^{[\tau]})-f\|_{\mathcal{X}}
\lesssim N^{-\alpha},
\qquad
f\in\mathcal{B}.
\end{equation}
\end{proposition}


\subsection{Polynomials}\label{ex:polynomial_bits}
Polynomial approximation is a fundamental tool in approximation theory, particularly effective for smooth functions due to the density of polynomials in $L^2$-spaces (see, e.g., \cite[Chapter 7]{devore1993constructive}). Orthogonal polynomials, such as Legendre polynomials, are especially useful because of their numerical stability \cite[Chapter 22]{stoer1980introduction} and convergence properties \cite[Chapter 5]{timan2014theory}.

Legendre polynomials form a complete orthonormal basis for $L^2([-1,1])$ (see, e.g., \cite{szego1939orthogonal,abramowitz1964handbook,davis1975interpolation}). We use the normalized Legendre polynomials satisfying $\int_{-1}^1 P_j(x)P_k(x)\,dx=\delta_{j,k}$ for $j,k\in\NN$, with multivariate extensions $\{P_\bk\}_{\bk\in\NN^d}$ defined via tensor products.

The following theorem provides approximation bounds using Legendre polynomials; the upper bound is classical \cite[Chapter 7]{devore1993constructive}, and the lower bound follows from width theory \cite{devore1989optimal}.

\begin{theorem}[Polynomial approximation: parameters]\label{ex:polynomial}
\it Let $r,d\in\NN$ and $\Omega=[-1,1]^d$. For any $f\in\mathcal{H}^{r}(\Omega)$, the orthogonal projection
$$T_n(f)=\sum\limits_{\|\bk\|_\infty<n}a_{\bk}(f)P_\bk\in\Pi_n^d$$
given by $a_\bk(f)=\left<P_\bk,f\right>_{L^2(\Omega)}$ satisfies
\begin{equation*}
    \sup\limits_{f\in\BB(\mathcal{H}^{r}(\Omega))}\|T_n(f)-f\|_{L^2(\Omega)}\simeq N^{-\frac{r}{d}},
\end{equation*}
where $N=n^d$ is the number of parameters.
\end{theorem}

\begin{corollary}[Polynomial approximation: bits]
    Let $\Nb\in\NN$, the operator $T_n$ in Theorem \ref{ex:polynomial} induces a binary bits approximation for functions in $\mathcal{H}^r(\Omega)$
    \begin{equation}
        \{0,1\}^{\Nb}\mapsto T_n(f)^{[\tau]}
    \end{equation}
    such that
    $$\sup\limits_{f\in\BB(\mathcal{H}^{r}(\Omega))}\lt\|f- T_n(f)^{[\tau]}\rt\|_{L^2(\Omega)}\lesssim\lrt{\frac{\Nb}{\log_2\Nb}}^{-\frac{r}{d}},$$
    and the corresponding lower bound is
    $$\inf\limits_{\psi,\phi}\sup\limits_{f\in\BB(\mathcal{H}^{r}(\Omega))}\lt\|\phi\circ\psi(f)-f\rt\|_{L^2(\Omega)}\gtrsim\Nb^{-\frac{r}{d}},$$
    where the infimum is taken over all maps $\psi:\BB(\mathcal{H}^{r}(\Omega))\to\{0,1\}^{\Nb}$ and $\phi:\{0,1\}^{\Nb}\to L^2(\Omega)$.
\end{corollary}

\begin{proof}
Let $f\in\BB(\mathcal{H}^r(\Omega))$ and $T_n(f)$ be as in Theorem \ref{ex:polynomial}. By Parseval's theorem,
$$\lt|a_{\bk}(f)\rt|\leq\|f\|_{L^2(\Omega)}\leq\|f\|_{\mathcal{H}^{r}(\Omega)}\leq1.$$
Denoting
$$S_N:~f\mapsto\{a_\bk(f)\}_{\|\bk\|_\infty<n},\quad\tilde S_N:\{a_\bk\}_{\|\bk\|_\infty<n}\mapsto\sum\limits_{\|\bk\|_\infty<n}a_{\bk}P_\bk$$
be as in \eqref{eqn_def_SN}, then by taking $M=1$, it is easy to show a crude bound
\begin{equation*}
    \mathrm{Lip}(M,\tilde S_N)_{L^2(\Omega)}\lesssim N.
\end{equation*}
By Proposition \ref{prop_para_to_bits}, taking $\tau=c_1+\lceil(1+\frac{r}{d})\log_2N\rceil$ for some constant $c_1$ independent of $N$, and setting $T_n(f)^{[\tau]}=\tilde S_N(S_N(f)^{[\tau]})$, we have
\begin{equation*}
	\begin{split}
		\lt\|f- T_n(f)^{[\tau]}\rt\|_{L^2(\Omega)}
		\lesssim N^{-\frac{r}{d}}\lesssim \lrt{\frac{\Nb}{\log_2\Nb}}^{-\frac{r}{d}},
	\end{split}
\end{equation*}
where $\Nb=\tau N$ is the number of binary bits that represents $T_n(f)^{[\tau]}$.

For the lower bound, we apply Proposition \ref{prop:entropy_equiv}. The operator $T_n^{[\tau]}$ can be written as $T_n^{[\tau]}=\phi\circ\psi$, where $\psi$ maps $f$ to the binary bits representation of the sequence $\displaystyle\left\{a_{\bk}(f)^{[\tau]}\right\}_{\|\bk\|_\infty<n}$ and $\phi$ maps the binary bits representation to the approximant $\displaystyle\sum\limits_{\|\bk\|_\infty<n}a_{\bk}(f)^{[\tau]}P_\bk$. Therefore, Theorem \ref{thm:entropy_sob} derives 
\begin{equation}
    \begin{split}
        \inf\limits_{\psi,\phi}\sup\limits_{f\in\BB(\mathcal{H}^{r}(\Omega))}\lt\|\phi\circ\psi(f)-f\rt\|_{L^2(\Omega)}\gtrsim\Nb^{-\frac{r}{d}},
    \end{split}
\end{equation}
where the infimum is taken over all $\psi:\BB(\mathcal{H}^{r}(\Omega))\to\{0,1\}^{\Nb}$ and $\phi:\{0,1\}^{\Nb}\to L^2(\Omega)$.

\end{proof}

\subsection{Finite elements}\label{ex:FE_2_bits}

Piecewise polynomial approximation, including the finite element method (FEM), is a fundamental approach for approximating functions with local regularity \cite[Chapter 3]{devore1993constructive}. It is based on partitioning the domain into smaller subdomains (e.g., simplices) and applying polynomial approximations locally, providing flexibility for handling complex geometries \cite[Chapter 4]{suli2003introduction}. As a standard framework for solving partial differential equations \cite{brenner2008mathematical}, FEM constructs finite-dimensional spaces of piecewise polynomials with prescribed continuity. Its convergence rates depend on both the smoothness of the target function and the polynomial degree \cite{bathe2006finite}, and it is widely used in engineering and physics applications \cite{moaveni2011finite,reddy1993introduction,jin2015finite}.

Let $\Omega = [0,1]^d$. For $n \in \mathbb{N}$, let $\mathcal{T}_n$ denote a uniform simplicial grid partitioning $\Omega$ into simplices of size $n^{-1}$. The set of nodal points $\mathcal{N}_n = \{\bx_\bk\}$ consists of the vertices of these simplices:
\begin{equation*}
\mathcal{N}_n = \left\{\bx_\bk = \left(\frac{k_1}{n}, \frac{k_2}{n}, \ldots, \frac{k_d}{n}\right) : \bk = (k_1, k_2, \ldots, k_d)\in \{0, 1, \ldots, n\}^d \right\}.
\end{equation*}
Thus the number of nodal degrees of freedom is $N=(n+1)^d$.

Let $\{\phi_{\bk,n}\}$ be the nodal basis functions that are piecewise linear and satisfy $\phi_{\bk_1,n}(\bx_{\bk_2}) = \delta_{\bk_1,\bk_2}$ for all $\bx_{\bk_1}, \bx_{\bk_2} \in \mathcal{N}_n$. We define the finite element space with $N=(n+1)^d$ degrees of freedom by
\begin{equation*}
V_N := \mathrm{span}\left(\left\{\phi_{\bk,n} : \bk \in \{0,1,\ldots,n\}^d \right\}\right),
\end{equation*}
then the finite element approximation $T_n(f) \in V_N$ of a function $f$ is obtained by solving the variational problem:
\begin{equation}\label{eqn:lin_eqn_FEM}
\lt<T_n(f),\phi_{\bm{j},n}\rt>=\lt<f,\phi_{\bm{j},n}\rt>, \quad \bm{j}\in\{0,\dots,n\}^d
\end{equation}

The following theorem provides approximation bounds, with the lower bound following from width theory \cite{devore1989optimal} and the upper bound from classical finite element theory \cite{bertoluzza2012primer,ciarlet2002finite}.

\begin{theorem}[Finite element approximation: parameters]\label{ex:FE_2}
\it Let $d,n$ and $\Omega$ be as above. For any $f \in \mathcal{H}^{2}(\Omega)$, let $T_n(f)\in V_N$ be the finite element approximation given by solving \eqref{eqn:lin_eqn_FEM}. Then
\begin{equation}
    \sup\limits_{f\in\BB(\mathcal{H}^2(\Omega))}\|f - T_n(f)\|_{L^2(\Omega)} \simeq  N^{-\frac{2}{d}}.
\end{equation}
where $N \simeq n^d$ is the number of parameters (degrees of freedom) in the finite element space $V_N$.
\end{theorem}

\begin{corollary}[Finite element approximation: bits]
    Let $\Nb\in\NN$, the operator $T_n$ in Theorem \ref{ex:FE_2} induces a binary bits approximation for functions in $\mathcal{H}^2(\Omega)$
    \begin{equation}
        \{0,1\}^{\Nb}\mapsto T_n(f)^{[\tau]}
    \end{equation}
    such that
    $$\sup\limits_{f\in\BB(\mathcal{H}^2(\Omega))}\lt\|f- T_n(f)^{[\tau]}\rt\|_{L^2(\Omega)}\lesssim\lrt{\frac{\Nb}{\log_2\Nb}}^{-\frac{2}{d}},$$
    and the corresponding lower bound is
    $$\inf\limits_{\psi,\phi}\sup\limits_{f\in\BB(\mathcal{H}^2(\Omega))}\lt\|\phi\circ\psi(f)-f\rt\|_{L^2(\Omega)}\gtrsim\Nb^{-\frac{2}{d}},$$
    where the infimum is taken over all maps $\psi:\BB(\mathcal{H}^{2}(\Omega))\to\{0,1\}^{\Nb}$ and $\phi:\{0,1\}^{\Nb}\to L^2(\Omega)$.
\end{corollary}

\begin{proof}
Let $f\in\BB(\mathcal{H}^2(\Omega))$ and
\[
T_n(f)=\sum_{\bk\in\NN^d,\;n^{-1}\bk\in\Omega} a_{\bk,n}(f)\phi_{\bk,n}.
\]
as in Theorem \ref{ex:FE_2}. Since
\[
\left\|\sum_{\bk}a_{\bk,n}(f)\phi_{\bk,n}\right\|_{L^2(\Omega)}
\le
\|f\|_{L^2(\Omega)}+\|T_n(f)-f\|_{L^2(\Omega)},
\]
for each coefficient we have
\[
|a_{\bk,n}(f)|\,\mathrm{dist}\!\left(\phi_{\bk,n},\mathrm{span}(\{\phi_{\bm j,n}\}_{\bm j\neq\bk})\right)
\lesssim
\|f\|_{L^2(\Omega)}+\|T_n(f)-f\|_{L^2(\Omega)}.
\]
Let $\Delta_{\bk,n}\in\mathcal{T}_n$, and consider a function of the form
\[
g=\phi_{\bk,n}+\sum_{\bm{j}\neq\bk}b_{\bm{j},n}\phi_{\bm{j},n}.
\]
Then $g(\frac{\bk}{n})=1$ and $g$ is linear on $\Delta_{\bk,n}$, which implies
\[
\|g\|_{L^2(\Delta_{\bk,n})}\gtrsim n^{-\frac{d}{2}}\simeq N^{-\frac{1}{2}}.
\]
Hence
\[
\mathrm{dist}\!\left(\phi_{\bk,n},\mathrm{span}(\{\phi_{\bm{j},n}\}_{\bm{j}\neq\bk})\right)\gtrsim N^{-\frac{1}{2}}.
\]
Combining the above estimates gives
\[
|a_{\bk,n}(f)|\lesssim\|f\|_{L^2(\Omega)}N^{\frac{1}{2}}\leq\|f\|_{\mathcal{H}^{2}(\Omega)}N^{\frac{1}{2}}.
\]

Denoting
$$S_N:~f\mapsto\{a_{\bk,n}(f)\}_{n^{-1}\bk\in\Omega},\quad\tilde S_N:\{a_{\bk,n}\}_{n^{-1}\bk\in\Omega}\mapsto\sum\limits_{n^{-1}\bk\in\Omega}a_{\bk,n}\phi_{\bk,n}$$
be as in \eqref{eqn_def_SN}, then with $M=\mathcal{O}(N)$, we can verify
\begin{equation*}
    \mathrm{Lip}(M,\tilde S_N)_{L^2(\Omega)}\lesssim N.
\end{equation*}
By Proposition \ref{prop_para_to_bits}, taking $\tau=c_2+\lceil (2+\frac{2}{d})\log_2N\rceil$ for some constant $c_2$ independent of $N$ and setting $T_n(f)^{[\tau]}=\tilde S_N(S_N(f)^{[\tau]})$,
\begin{equation*}
	\begin{split}
		\lt\|f- T_n(f)^{[\tau]}\rt\|_{L^2(\Omega)}
		\lesssim N^{-\frac{2}{d}}\lesssim\lrt{\frac{\Nb}{\log_2\Nb}}^{-\frac{2}{d}},
	\end{split}
\end{equation*}
where $\Nb=\tau N$ is the number of binary bits that represents $T_n(f)^{[\tau]}$.

For the lower bound, by Proposition \ref{prop:entropy_equiv} and Theorem \ref{thm:entropy_sob},
\begin{equation}
    \begin{split}
        \inf\limits_{\psi,\phi}\sup\limits_{f\in\BB(\mathcal{H}^{2}(\Omega))}\lt\|\phi\circ\psi(f)-f\rt\|_{L^2(\Omega)}\gtrsim\Nb^{-\frac{2}{d}},
    \end{split}
\end{equation}
where the infimum is taken over all $\psi:\BB(\mathcal{H}^{2}(\Omega))\to\{0,1\}^{\Nb}$ and $\phi:\{0,1\}^{\Nb}\to L^2(\Omega)$.

\end{proof}

Adaptive finite element methods (AFEM) improve upon standard FEM by selectively refining the mesh in regions where the approximation error is large. For functions $f$ in Besov spaces $B^s_{p,q}(\Omega)$ with smoothness parameter $s > d/p$, classical AFEM theory shows that the approximation error satisfies
$$\sup\limits_{f\in\BB(B^s_{p,q}(\Omega))}\|f - T_n(f)\|_{L^2(\Omega)} \lesssim N^{-\frac{s}{d}},$$
where $N$ is the number of parameters. This adaptive approach is particularly effective for functions with localized features, achieving optimal convergence rates with fewer elements than uniform refinement.

\subsection{Finitely smooth kernels}\label{ex:kernel_bits}

Kernel-based approximation provides a flexible nonparametric framework for approximating functions on general domains. Let $\Omega\subset\RR^d$. A translation-invariant kernel on $\Omega\times\Omega$ is generated by a function $\Phi:\RR^d\to\RR$ and takes the form
\[
    K(\bx,\by)=\Phi(\bx-\by),\qquad \bx,\by\in\Omega.
\]
We consider finitely smooth kernels of this form, where $\Phi\in L^2(\RR^d)$ is a continuous positive definite function whose Fourier transform satisfies
\begin{equation}\label{eqn:finitely_smooth_kernel}
    \widehat{\Phi}(\bm\omega)
    \simeq(1+\|\bm\omega\|_2^2)^{-s},
    \qquad \bm\omega\in\mathbb{R}^d,
\end{equation}
for some $s>d/2$. Examples include Matérn kernels and Wendland kernels.

Given a set of centers $X=\{\bx_j\}_{j=1}^N\subset\Omega$, the kernel interpolant of $f$ is the linear combination
\begin{equation}\label{eqn_def_TNf_kernel}
    T_N(f)(\bx)=\sum_{j=1}^N a_j(f)\Phi(\bx-\bx_j),
\end{equation}
where the coefficients $\{a_j(f)\}_{j=1}^N$ are determined by the interpolation conditions
\begin{equation*}
    \sum_{j=1}^N a_j(f)\Phi(\bx_i-\bx_j)=f(\bx_i),
    \qquad i=1,\dots,N.
\end{equation*}
For quasi-uniform centers, the fill distance satisfies $h_X\simeq N^{-1/d}$. Approximation results for finitely smooth kernels have been extensively studied in the literature \cite{schaback1999improved,mhaskar1999zonal,mhaskar2004kernel,narcowich2005sobolev,wendland2005approximate,narcowich2006sobolev,arcangeli2007extension,mhaskar2010eignets,schaback2018superconvergence,mhaskar2020kernel,karvonen2025superconvergence,sloan2025doubling}. We state the result in terms of the number of centers as follows; see, e.g., \cite[Theorem 1]{wenzel2026sharp}.

\begin{theorem}[Finitely smooth kernel approximation: parameters]\label{ex:kernel}
\it Let $\Omega\subset\mathbb{R}^d$ be a bounded Lipschitz domain, let $X=\{\bx_j\}_{j=1}^N\subset\Omega$ be quasi-uniform, and let $\Phi$ satisfy \eqref{eqn:finitely_smooth_kernel} with $s>d/2$. Then for $r\in(d/2,2s]$ and $f\in \mathcal H^r(\Omega)$, the kernel interpolant $T_N(f)$ in \eqref{eqn_def_TNf_kernel} satisfies
\begin{equation}
    \sup_{f\in\BB(\mathcal H^r(\Omega))}
    \|f-T_N(f)\|_{L^2(\Omega)}
    \lesssim N^{-\frac{r}{d}}.
\end{equation}
\end{theorem}

\begin{corollary}[Finitely smooth kernel approximation: bits]
Let $\Nb\in\NN$. The operator $T_N$ in the above theorem induces a binary bits approximation
\begin{equation}
    \{0,1\}^{\Nb}\mapsto T_N(f)^{[\tau]}
\end{equation}
such that
\begin{equation}
    \sup_{f\in\BB(\mathcal H^r(\Omega))}
    \left\|f-T_N(f)^{[\tau]}\right\|_{L^2(\Omega)}
    \lesssim
    \left(\frac{\Nb}{\log_2\Nb}\right)^{-\frac{r}{d}},
\end{equation}
and the corresponding lower bound is
\begin{equation}
    \inf_{\psi,\phi}
    \sup_{f\in\BB(\mathcal H^r(\Omega))}
    \|\phi\circ\psi(f)-f\|_{L^2(\Omega)}
    \gtrsim \Nb^{-\frac{r}{d}},
\end{equation}
where the infimum is taken over all maps $\psi:\BB(\mathcal H^r(\Omega))\to\{0,1\}^{\Nb}$ and $\phi:\{0,1\}^{\Nb}\to L^2(\Omega)$.
\end{corollary}

\begin{proof}
Let $f\in\BB(\mathcal H^r(\Omega))$ and let $T_N(f)$ be as in \eqref{eqn_def_TNf_kernel}. Denote
\[
    S_N:f\mapsto \{a_j(f)\}_{j=1}^N,\qquad
    \tilde S_N:\{a_j\}_{j=1}^N\mapsto \sum_{j=1}^N a_j\Phi(\cdot-\bx_j),
\]
as in \eqref{eqn_def_SN}. By the standard coefficient stability estimate for finitely smooth kernel approximation (see, e.g., \cite[Proposition 18]{wenzel2026sharp}), we have
\begin{equation*}
    \|S_N(f)\|_\infty\leq\|S_N(f)\|_2
    \lesssim N^{\frac{2s}{d}-\frac{1}{2}}\|T_N(f)\|_{L^2(\Omega)}
    \lesssim N^{\frac{2s}{d}-\frac{1}{2}},
\end{equation*}
where the last inequality follows from
\[
    \|T_N(f)\|_{L^2(\Omega)}
    \leq \|f\|_{L^2(\Omega)}+\|f-T_N(f)\|_{L^2(\Omega)}
    \lesssim 1.
\]
Moreover, with $M=\mathcal{O}(N^{\frac{2s}{d}-\frac{1}{2}})$, it is straightforward to verify that
\[
    \mathrm{Lip}(M,\tilde S_N)_{L^2(\Omega)}\lesssim N.
\]
By Proposition \ref{prop_para_to_bits}, $\tau=c_3+\left\lceil\left(\frac{2s+r}{d}+\frac{1}{2}\right)\log_2N\right\rceil$ for some constant $c_3$ independent of $N$, and setting $T_N(f)^{[\tau]}=\tilde S_N(S_N(f)^{[\tau]})$, we have
\[
    \left\|f-T_N(f)^{[\tau]}\right\|_{L^2(\Omega)}
    \lesssim N^{-\frac{r}{d}}
    \lesssim
    \left(\frac{\Nb}{\log_2\Nb}\right)^{-\frac{r}{d}},
\]
where $\Nb=\tau N$ is the number of binary bits that represents $T_N(f)^{[\tau]}$.

For the lower bound, by Proposition \ref{prop:entropy_equiv} and Theorem \ref{thm:entropy_sob},
\[
    \inf_{\psi,\phi}
    \sup_{f\in\BB(\mathcal H^r(\Omega))}
    \|\phi\circ\psi(f)-f\|_{L^2(\Omega)}
    \gtrsim \Nb^{-\frac{r}{d}}.
\]
\end{proof}

\subsection{Sparse grids}\label{ex:sparsegrid_bits}
While adaptive finite element methods (AFEM) improve efficiency through local refinement, the curse of dimensionality remains a major challenge in high-dimensional settings. This motivates alternative approaches such as the sparse grid method \cite{bungartz2004sparse}, which employs hierarchical basis constructions to reduce the exponential growth of complexity compared to full tensor product schemes. Sparse grids are particularly effective for functions with bounded mixed derivatives, achieving high accuracy with significantly fewer degrees of freedom.

Let $\Omega = [0,1]^d$, the Korobov space is a space defined by mixed derivatives as
\begin{equation}
    \mathcal{U}^2(\Omega):=\{f\in L^2(\Omega):~\|f\|_{\mathcal{U}^2(\Omega)}<\infty\},
\end{equation}
where
\begin{equation}\label{eqn:def_Koro_norm}
    \|f\|_{\mathcal{U}^2(\Omega)}^2=\|f\|_{L^2(\Omega)}^2+\sul_{\bm{s}\in\NN^d,1\leq\|\bm s\|_\infty\leq2}\|D^{\bm s}f\|_{L^2(\Omega)}^2.
\end{equation}

Sparse grid approximation provides an efficient alternative to standard tensor product approaches for functions in $\mathcal{U}^2(\Omega)$. The sparse grid construction uses hierarchical basis functions obtained through tensor products of a univariate mother function
$$\psi(x)=\left\{\begin{array}{ll}
    1-|x|, &\quad\hbox{if }x\in[-1,1],  \\
    0, &\quad\hbox{otherwise.}
\end{array}\right.$$
Let $\bm{I}_\bl$ denote the index set:
$\bm{I}_\bl := \{\bk\in\NN^d:~\bm{1}\leq\bk\leq2^{\bl}-\bm{1},~k_j\hbox{ odd for all }j\}$. For multi-indices $\bl \in \mathbb{N}^d$ and $\bk \in \mathbb{N}^d$, define the hierarchical basis functions:
\begin{equation}\label{eqn:hierarchical_basis}
    \psi_{\bl,\bk}=\prod\limits_{j=1}^d\psi\lrt{2^{l_j}x_j-k_j},\qquad x\in\Omega.
\end{equation}
 $f$ is approximated by the sparse grid function
\begin{equation}\label{eqn:def_sparse_appr}
    T_n(f)=\sul_{\|\bl\|_1\leq n+d-1}\sul_{\bk\in\bm{I}_\bl}a_{\bl,\bk}\psi_{\bl,\bk},
\end{equation}
where
\begin{equation}\label{eqn:coef_korobov}
    a_{\bl,\bk}=\int_\Omega\prod\limits_{j=1}^d\lrt{-2^{-l_j-1}\psi\lrt{2^{l_j}x_j-k_j}}\frac{\partial^{2d}f}{\partial x_1^2\dots\partial x_d^2}(\bx)d\bx,\quad \bl\in\NN^d~\bk\in\bm{I}_\bl.
\end{equation}
The following lower bound follows again from width theory \cite{devore1989optimal} and the upper bound from classical finite element theory \cite{bungartz2004sparse}.
\begin{theorem}[Sparse grid approximation: parameters]\label{ex:sparsegrid}
\it Let $d,n$ and $\Omega$ be as above. For any $f \in \mathcal{U}^2(\Omega)$, let $T_n(f)$ be the function given by \eqref{eqn:def_sparse_appr}, then
\begin{equation}
    N^{-2}\lesssim \sup\limits_{f\in\BB(\mathcal{U}^2(\Omega))}\lt\|f-T_n(f)\rt\|_{L^2(\Omega)}\lesssim N^{-2}(\log N)^{3(d-1)},
\end{equation}
where $N:=\sul_{\|\bl\|_1\leq n+d-1}\#\bm{I}_\bl$ denotes the total number of parameters in the approximation.
\end{theorem}

\begin{corollary}[Sparse grid approximation: bits]
    Let $\Nb\in\NN$, the operator $T_n$ in Theorem \ref{ex:sparsegrid} induces a binary bits approximation for functions in $\mathcal{U}^{2}(\Omega)$
    \begin{equation}
        \{0,1\}^{\Nb}\mapsto T_n(f)^{[\tau]}
    \end{equation}
    such that
    $$\sup\limits_{f\in\BB(\mathcal{U}^2(\Omega))}\lt\|f- T_n(f)^{[\tau]}\rt\|_{L^2(\Omega)}\lesssim\Nb^{-2}(\log\Nb)^{3d-1},$$
    and the corresponding lower bound is
    $$\inf\limits_{\psi,\phi}\sup\limits_{f\in\BB(\mathcal{U}^2(\Omega))}\lt\|\phi\circ\psi(f)-f\rt\|_{L^2(\Omega)}\gtrsim\Nb^{-2},$$
    where the infimum is taken over all maps $\psi:\BB(\mathcal{U}^{2}(\Omega))\to\{0,1\}^{\Nb}$ and $\phi:\{0,1\}^{\Nb}\to L^2(\Omega)$.
\end{corollary}
\begin{proof}

In Theorem \ref{ex:sparsegrid}, given function in $\BB(\mathcal{U}^2(\Omega))$, the approximation rate in terms of number of parameters is $\|T_n(f)-f\|_{L^2(\Omega)}=\mathcal{O}(N^{-2}(\log N)^{3(d-1)})$.

    By \eqref{eqn:coef_korobov}, the coefficients of $a_{\bl,\bk}$ are bounded by
	$$2^{-2\|\bl\|_1-d}\lt\|\frac{\partial^{2d}f}{\partial x_1^2\dots\partial x_d^2}\rt\|_{L^2(\Omega)}\leq2^{-2\|\bl\|_1-d}\|f\|_{\mathcal{U}^2(\Omega)}.$$
    Denoting
$$S_N:~f\mapsto\{a_{\bl,\bk}(f)\}_{\substack{\|\bl\|_1\leq n+d-1\\
\bk\in\bm{I}_\bl}},\quad\tilde S_N:\{a_{\bl,\bk}\}_{\substack{\|\bl\|_1\leq n+d-1\\
\bk\in\bm{I}_\bl}}\mapsto\sul_{\|\bl\|_1\leq n+d-1}\sul_{\bk\in\bm{I}_\bl}a_{\bl,\bk}\psi_{\bl,\bk}$$
be as in \eqref{eqn_def_SN}, then by taking $M=1$ we can verify
\begin{equation*}
    \mathrm{Lip}(M,\tilde S_N)_{L^2(\Omega)}\lesssim N.
\end{equation*}
By Proposition \ref{prop_para_to_bits}, taking $\tau=c_4+\lceil3\log_2N\rceil$ for some constant $c_4$ independent of $N$, and setting $T_n(f)^{[\tau]}=\tilde S_N(S_N(f)^{[\tau]})$,
\begin{equation}\label{eqn_koro_bits}
	\begin{split}
		\lt\|f- T_n(f)^{[\tau]}\rt\|_{L^2(\Omega)}
		\lesssim N^{-2}(\log N)^{3(d-1)}\lesssim\Nb^{-2}(\log \Nb)^{3d-1},
	\end{split}
\end{equation}
where $\Nb=\tau N$ is the number of binary bits that represents $T_n(f)^{[\tau]}$.

The metric entropy of is estimated as
\begin{equation}
    m^{-2}\lesssim\epsilon_m(\BB(\mathcal{U}^2(\Omega)))\lesssim m^{-2}(\log m)^{3d-1}.
\end{equation}
where the upper bound is given by Proposition \ref{prop:entropy_equiv} and \eqref{eqn_koro_bits}, the lower bound is given by Theorem \ref{thm:entropy_sob} and the fact
$\BB(\mathcal{H}^{2d}(\Omega))\subset\BB(\mathcal{U}^2(\Omega))$.

By Proposition \ref{prop:entropy_equiv}, we have
\begin{equation}
    \begin{split}
        \inf\limits_{\psi,\phi}\sup\limits_{f\in\BB(\mathcal{U}^2(\Omega))}\lt\|\phi\circ\psi(f)-f\rt\|_{L^2(\Omega)}\gtrsim\Nb^{-2},
    \end{split}
\end{equation}
where the infimum is taken over all $\psi:\BB(\mathcal{U}^2(\Omega))\to\{0,1\}^{\Nb}$ and $\phi:\{0,1\}^{\Nb}\to L^2(\Omega)$.
\end{proof}

The sparse grid construction attains fast approximation rates in terms of both the number of parameters and the number of bits. However, the metric entropy of $\mathcal{U}^2(\Omega)$ decays at essentially the same rate as that of the Sobolev space $\mathcal{H}^{2d}(\Omega)$, for which the classical polynomial approximation also achieves the rate $\mathcal{O}(N^{-2})$ in terms of parameters and $\mathcal{O}\!\left(\frac{\Nb}{\log_2 \Nb}\right)^{-2}$ in terms of bits.

\section{Neural networks approximation in terms of number of parameters and bits}\label{sec:appr_nn}

\subsection{Shallow neural networks for Barron spaces}\label{ex:snn_L2_bits}
Shallow neural networks, particularly those using ReLU and its higher-order variants, have become key tools in function approximation and machine learning, especially in high-dimensional settings. Their effectiveness stems from the fact that neural networks with sufficiently many neurons can approximate smooth functions, as well as functions with more specialized structures such as Barron-type regularity, at nearly optimal rates.

    \begin{definition}[Shallow neural networks]\label{def:snn}
    Let $n\in\NN$, $M>0$, a class of ReLU$^k$ neural network is given by
    \begin{equation}
	\Sigma_n^{k}:=\lt\{\bx\mapsto\sum\limits_{j=1}^na^{(j)}\sigma_k(\bw^{(j)}\cdot\bx+b^{(j)}):~\binom{\bw^{(j)}}{b^{(j)}}\in\mathbb{S}^{d},~a^{(j)}\in\RR\rt\}
    \end{equation}
    If we further assume a bound of the coefficients, we denote
    \begin{equation}
	\Sigma_n^{k}(M):=\lt\{\bx\mapsto\sum\limits_{j=1}^na^{(j)}\sigma_k(\bw^{(j)}\cdot\bx+b^{(j)}):~\binom{\bw^{(j)}}{b^{(j)}}\in\mathbb{S}^{d},~\sul_{j=1}^n|a^{(j)}|\leq M\rt\}
    \end{equation}
    \end{definition}

\begin{theorem}[Shallow neural networks on Barron spaces: parameters \cite{siegel2022sharp}]\label{ex:snn_L2}
\it Let $d,k,n,\Omega$ be as above, $\mathcal{B}^k(\Omega)$ be as in Definition \ref{def:variation}. There exists a constant $C_1=C_1(d,k)$ such that the map
$$T_n(f):=\arg\min\limits_{g\in\Sigma_n^k(C_1)}\|f-g\|_{L^2(\Omega)},\quad f\in\BB(\mathcal{B}^{k}(\Omega))$$
satisfies
\begin{equation}\label{eqn:variation_appr}
    (N\log N)^{-\frac{1}{2}-\frac{2k+1}{2d}}\lesssim \sup\limits_{f\in\BB(\mathcal{B}^{k}(\Omega))}\|T_n(f)-f\|_{L^2(\Omega)}\lesssim N^{-\frac{1}{2}-\frac{2k+1}{2d}},
\end{equation}
where $N=(d+2)n$ is the number of parameters.
\end{theorem}

\begin{corollary}[Shallow neural networks on Barron spaces: bits]
    Let $\Nb\in\NN$, the operator $T_n$ in Theorem \ref{ex:snn_L2} induces a binary bits approximation for functions in $\mathcal{B}^k(\Omega)$
    \begin{equation}
        \{0,1\}^{\Nb}\mapsto T_n(f)^{[\tau]}
    \end{equation}
    such that
    $$\sup\limits_{f\in\BB(\mathcal{B}^{k}(\Omega))}\lt\|f- T_n(f)^{[\tau]}\rt\|_{L^2(\Omega)}\lesssim\lrt{\frac{\Nb}{\log_2\Nb}}^{-\frac{1}{2}-\frac{2k+1}{2d}},$$
    and the corresponding lower bound is
    $$\inf\limits_{\psi,\phi}\sup\limits_{f\in\BB(\mathcal{B}^{k}(\Omega))}\lt\|\phi\circ\psi(f)-f\rt\|_{L^2(\Omega)}\gtrsim\Nb^{-\frac{1}{2}-\frac{2k+1}{2d}},$$
    where the infimum is taken over all $\psi:\BB(\mathcal{B}^k(\Omega))\to\{0,1\}^{\Nb}$ and $\phi:\{0,1\}^{\Nb}\to L^2(\Omega)$.
\end{corollary}
\begin{proof}
For $f\in\BB(\mathcal{B}^k(\Omega))$, let
	$$T_n(f)(\bx)=\sul_{j=1}^na^{(j)}(f)\sigma_k\left(\bw^{(j)}(f)\cdot\bx+b^{(j)}(f)\right)$$
be the function in Theorem \ref{ex:snn_L2} with $\|T_n(f)-f\|_{L^2(\Omega)}=\mathcal{O}(N^{-\frac{1}{2}-\frac{2k+1}{2d}})$. Denoting
\begin{equation*}
    \begin{split}
        &S_N:~f\mapsto\left\{(a^{(j)}(f),\bw^{(j)}(f),b^{(j)}(f))\right\}_{j=1}^n,\\
        &\tilde S_N:\left\{(a^{(j)},\bw^{(j)},b^{(j)})\right\}_{j=1}^n\mapsto\sul_{j=1}^na^{(j)}\sigma_k\left(\bw^{(j)}\cdot\bx+b^{(j)}\right)
    \end{split}
\end{equation*}
be as in \eqref{eqn_def_SN}. Then by taking $M=C_1$ and we can verify
\begin{equation*}
    \mathrm{Lip}(M,\tilde S_N)_{L^2(\Omega)}\lesssim N.
\end{equation*}
By Proposition \ref{prop_para_to_bits}, taking $\tau=c_5+\lceil(\frac{3}{2}+\frac{2k+1}{2d})\log_2N\rceil$ for some constant $c_5$ independent of $N$, and setting $T_n(f)^{[\tau]}=\tilde S_N(S_N(f)^{[\tau]})$,
\begin{equation*}
	\begin{split}
		\lt\|f- T_n(f)^{[\tau]}\rt\|_{L^2(\Omega)}
		\lesssim N^{-\frac{1}{2}-\frac{2k+1}{2d}}\lesssim\lrt{\frac{\log_2\Nb}{\Nb}}^{\frac{1}{2}+\frac{2k+1}{2d}},
	\end{split}
\end{equation*}
where $\Nb=\tau N$ is the number of binary bits that represents $T_n(f)^{[\tau]}$.

For the lower bound, by Proposition \ref{prop:entropy_equiv} and Theorem \ref{thm:entropy_variation},
\begin{equation}
    \begin{split}
        &\inf\limits_{\psi,\phi}\sup\limits_{f\in\BB(\mathcal{B}^{k}(\Omega))}\lt\|\phi\circ\psi(f)-f\rt\|_{L^2(\Omega)}\gtrsim\Nb^{-\frac{1}{2}-\frac{2k+1}{2d}}
    \end{split}
\end{equation}
where the infimum is taken over all $\psi:\BB(\mathcal{B}^{k}(\Omega))\to\{0,1\}^{\Nb}$ and $\phi:\{0,1\}^{\Nb}\to L^2(\Omega)$.

\end{proof}

\subsection{Linearized shallow neural networks for Sobolev spaces}\label{ex:lsnn_L2_bits}

We next consider a linearized version of shallow ReLU$^k$ networks, where the inner parameters are fixed in advance and only the outer coefficients are optimized.

\begin{definition}[Linearized shallow neural networks]\label{def:lsnn}
Let $N\in\NN$, $M>0$, and let $\{\theta_j^*\}_{j=1}^N=\big\{\binom{\bw^{(j)}}{b^{(j)}}\big\}_{j=1}^N\subset\SS^d$ be a collection of quasi-uniform points. The class of linearized ReLU$^k$ neural networks associated with these fixed inner parameters is given by
\begin{equation}
	L_N^k=L_N^k(\{\theta_j^*\}_{j=1}^N):=\lt\{\bx\mapsto\sum\limits_{j=1}^Na^{(j)}\sigma_k(\bw^{(j)}\cdot\bx+b^{(j)}):~a^{(j)}\in\RR\rt\}.
\end{equation}
If we further assume a bound on the outer coefficients, we denote
\begin{equation}
	L_N^k(M):=\lt\{\bx\mapsto\sum\limits_{j=1}^Na^{(j)}\sigma_k(\bw^{(j)}\cdot\bx+b^{(j)}):~N\sul_{j=1}^N|a^{(j)}|^2\leq M\rt\}.
\end{equation}
\end{definition}

Since the inner parameters are fixed, $L_N^k$ is an $N$-dimensional linear space. Hence the corresponding approximation problem is linear in the unknown coefficients and can be solved by convex optimization or, in Hilbert space settings, by solving a linear system. Related sharp approximation results for fixed inner parameters were established in \cite{petrushev1998approximation}, where the parameters are chosen as tensor products of quadrature points on $\SS^{d-1}$ and $[-1,1]$.

\begin{theorem}[Linearized shallow neural networks on Sobolev spaces: parameters \cite{liu2025achieving}]\label{ex:lsnn_L2}
\it Let $d,k,N,\Omega$ be as above. There exists a constant $C_2=C_2(d,k)$ such that the map
$$T_N(f):=\arg\min\limits_{g\in L_N^k(C_2)}\|f-g\|_{L^2(\Omega)},\quad f\in\mathcal{H}^{\frac{d+2k+1}{2}}(\Omega)$$
satisfies
\begin{equation}\label{eqn:sobolev_appr}
    (N\log N)^{-\frac{1}{2}-\frac{2k+1}{2d}}\lesssim \sup\limits_{f\in\BB\big(\mathcal{H}^{\frac{d+2k+1}{2}}(\Omega)\big)}\|T_N(f)-f\|_{L^2(\Omega)}\lesssim N^{-\frac{1}{2}-\frac{2k+1}{2d}},
\end{equation}
where $N$ is the number of fixed inner parameters, or equivalently the number of free outer coefficients.
\end{theorem}

Theorem~\ref{ex:lsnn_L2} also provides a useful perspective on saturation. In classical finite element approximation, once the polynomial degree is fixed, increasing the smoothness of the target function does not improve the convergence order beyond the saturation rate. For example, piecewise linear finite elements satisfy
\begin{equation}
    \inf_{g\in V_N}\|f-g\|_{L^2(\Omega)}\gtrsim N^{-\frac{2}{d}}
\end{equation}
for sufficiently smooth nontrivial functions; see \cite{lin2014lower}. By contrast, the linearized ReLU$^k$ spaces in Theorem~\ref{ex:lsnn_L2} can continue to reflect the Sobolev smoothness of the target class within the admissible range of the theorem. Thus, the absence of saturation is a parameter-based phenomenon associated with using richer approximation spaces. From the bit-complexity viewpoint, however, the achievable accuracy is still governed by the metric entropy of the target class.

\begin{corollary}[Linearized shallow neural networks on Sobolev spaces: bits]
    Let $\Nb\in\NN$. The operator $T_N$ in Theorem \ref{ex:lsnn_L2} induces a binary bits approximation for functions in $\BB\big(\mathcal{H}^{\frac{d+2k+1}{2}}(\Omega)\big)$
    \begin{equation}
        \{0,1\}^{\Nb}\mapsto T_N(f)^{[\tau]}
    \end{equation}
    such that
    $$\lt\|f- T_N(f)^{[\tau]}\rt\|_{L^2(\Omega)}\lesssim\|f\|_{\mathcal{H}^{\frac{d+2k+1}{2}}(\Omega)}\lrt{\frac{\Nb}{\log_2\Nb}}^{-\frac{1}{2}-\frac{2k+1}{2d}},$$
    and the corresponding lower bound is
    $$\inf\limits_{\psi,\phi}\sup\limits_{f\in\BB\big(\mathcal{H}^{\frac{d+2k+1}{2}}(\Omega)\big)}\lt\|\phi\circ\psi(f)-f\rt\|_{L^2(\Omega)}\gtrsim\Nb^{-\frac{1}{2}-\frac{2k+1}{2d}},$$
    where the infimum is taken over all $\psi:\BB(\mathcal{H}^{\frac{d+2k+1}{2}}(\Omega))\to\{0,1\}^{\Nb}$ and $\phi:\{0,1\}^{\Nb}\to L^2(\Omega)$.
\end{corollary}

\begin{proof}
In Theorem \ref{ex:lsnn_L2}, for functions in $\mathcal{H}^{\frac{d+2k+1}{2}}(\Omega)$, the approximation rate in terms of the number of free coefficients is
$$\|T_N(f)-f\|_{L^2(\Omega)}=\mathcal{O}\lrt{N^{-\frac{1}{2}-\frac{2k+1}{2d}}}.$$

Since the inner parameters are fixed, we write
$$T_N(f)(\bx)=\sul_{j=1}^Na^{(j)}(f)\sigma_k(\bw^{(j)}\cdot\bx+b^{(j)}).$$
Denote
$$S_N:f\mapsto \{a^{(j)}(f)\}_{j=1}^N,\qquad \widetilde S_N:\{a^{(j)}\}_{j=1}^N\mapsto \sul_{j=1}^Na^{(j)}\sigma_k(\bw^{(j)}\cdot\bx+b^{(j)}).$$
By the coefficient constraint in $L_N^k(C_2)$, we have
$$\|S_N(f)\|_{\infty}\leq \|S_N(f)\|_2\leq C_2^{\frac{1}{2}}N^{-\frac{1}{2}}.$$
Then by taking $M=C_2^{\frac{1}{2}}$, we can verify
\begin{equation*}
    \mathrm{Lip}(M,\tilde S_N)_{L^2(\Omega)}\lesssim N.
\end{equation*}
Therefore, by Proposition \ref{prop_para_to_bits}, taking
$\tau=c_6+\lt\lceil\lrt{\frac{3}{2}+\frac{2k+1}{2d}}\log_2N\rt\rceil$
for some $c_6$ independent of $N$, and setting $T_N(f)^{[\tau]}=\widetilde S_N(S_N(f)^{[\tau]})$, we obtain
$$\lt\|f- T_N(f)^{[\tau]}\rt\|_{L^2(\Omega)}\lesssim\|f\|_{\mathcal{H}^{\frac{d+2k+1}{2}}(\Omega)}\lrt{\frac{\Nb}{\log_2\Nb}}^{-\frac{1}{2}-\frac{2k+1}{2d}},$$
where $\Nb=\tau N$.

For the lower bound, we apply Proposition \ref{prop:entropy_equiv} and Theorem \ref{thm:entropy_sob} and conclude
\begin{equation}
    \inf\limits_{\psi,\phi}\sup\limits_{f\in\BB\big(\mathcal{H}^{\frac{d+2k+1}{2}}(\Omega)\big)}\lt\|\phi\circ\psi(f)-f\rt\|_{L^2(\Omega)}\gtrsim\Nb^{-\frac{1}{2}-\frac{2k+1}{2d}},
\end{equation}
where the infimum is taken over all $\psi:\BB(\mathcal{H}^{\frac{d+2k+1}{2}}(\Omega))\to\{0,1\}^{\Nb}$ and $\phi:\{0,1\}^{\Nb}\to L^2(\Omega)$.
\end{proof}

\subsection{Shallow neural networks for spectral Barron spaces}\label{ex:barron_bits}
The rate of approximation of spectral Barron spaces has been studied in \cite{barron1991complexity,klusowski2018approximation} and then generalized in \cite{ma2022uniform}.

\begin{theorem}[Shallow neural networks on spectral Barron spaces: parameters \cite{ma2022uniform}]\label{ex:barron}
    \it Let $d,k,n,\Omega$ be as above, $s=(d+1)(k+1)$, $\mathcal{V}^s(\Omega)$ be as in Definition \ref{def:spectralbarron}. There exists a constant $C_3$ such that, with $M=C_3n^{\frac{k+1}{d+1}}\log n$, the map
$$T_n(f):=\arg\min\limits_{g\in\Sigma_n^k(M)}\|f-g\|_{L^2(\Omega)},\quad f\in\BB(\mathcal{V}^s(\Omega))$$
satisfies
\begin{equation}\label{eqn:barron_appr}
    \sup\limits_{f\in\BB(\mathcal{V}^s(\Omega))}\|T_n(f)-f\|_{L^2(\Omega)}\lesssim N^{-(k+1)}\log N,
\end{equation}
where $N=(d+2)n$ is the number of parameters.
\end{theorem}

\begin{corollary}[Shallow neural networks on spectral Barron spaces: bits]
    Let $\Nb\in\NN$, the operator $T_n$ in Theorem \ref{ex:barron} induces a binary bits approximation for functions in $\BB(\mathcal{V}^{s}(\Omega))$
    \begin{equation}
        \{0,1\}^{\Nb}\mapsto T_n(f)^{[\tau]}
    \end{equation}
    such that
    $$\sup\limits_{f\in\BB(\mathcal{V}^s(\Omega))}\lt\|f- T_n(f)^{[\tau]}\rt\|_{L^2(\Omega)}\lesssim\lrt{\frac{\Nb}{\log_2\Nb}}^{-(k+1)}\log\Nb,$$
    and the corresponding lower bound is
    $$\inf\limits_{\psi,\phi}\sup\limits_{f\in\BB(\mathcal{V}^s(\Omega))}\lt\|\phi\circ\psi(f)-f\rt\|_{L^2(\Omega)}\gtrsim\Nb^{-(k+1)\frac{d+1}{d}-\frac{1}{2}},$$
    where the infimum is taken over all $\psi:\BB(\mathcal{V}^s(\Omega))\to\{0,1\}^{\Nb}$ and $\phi:\{0,1\}^{\Nb}\to L^2(\Omega)$.
\end{corollary}
\begin{proof}

For $f\in\BB(\mathcal{V}^s(\Omega))$, let
	$$T_n(f)(\bx)=\sul_{j=1}^na^{(j)}(f)\sigma_k\left(\bw^{(j)}(f)\cdot\bx+b^{(j)}(f)\right)$$
be the function in Theorem \ref{ex:barron} with $\|T_n(f)-f\|_{L^2(\Omega)}=\mathcal{O}(N^{-(k+1)}\log N)$. 
A same argument of Section \ref{ex:snn_L2_bits} yields the upper bound
$$\lt\|f- T_n(f)^{[\tau]}\rt\|_{L^2(\Omega)}\lesssim\lrt{\frac{\Nb}{\log_2\Nb}}^{-(k+1)}\log\Nb,$$
where $\Nb=\tau N$ and $\tau=c_7+\lceil(\frac{k+1}{d+1}+k+3)\log_2N\rceil$ for some $c_7$ independent of $N$.

For the lower bound, by Proposition \ref{prop:entropy_equiv} and Theorem \ref{thm:entropy_variation},
\begin{equation}
    \begin{split}
        \inf\limits_{\psi,\phi}\sup\limits_{f\in\BB(\mathcal{V}^s(\Omega))}\lt\|\phi\circ\psi(f)-f\rt\|_{L^2(\Omega)}\gtrsim\Nb^{-(k+1)\frac{d+1}{d}-\frac{1}{2}}
    \end{split}
\end{equation}
where the infimum is taken over all $\psi:\BB(\mathcal{V}^s(\Omega))\to\{0,1\}^{\Nb}$ and $\phi:\{0,1\}^{\Nb}\to L^2(\Omega)$.

\end{proof}

\subsection{Deep neural networks for Sobolev spaces}\label{ex:fnn_bits}
Deep ReLU networks refer to neural networks with multiple layers of neurons, where the activation function in each layer is typically the ReLU  function. The key distinction between shallow and deep networks is that deep networks involve several hidden layers, which allows them to learn more complex patterns in the data \cite{bengio2009learning,eldan2016power,poggio2017and}.

\begin{definition}[Fully-connected deep ReLU$^k$ networks]\label{def:dnn}
Denote by $\Sigma_{W,n}^k$ the class of all functions on $\Omega$ that can be constructed by the fully-connected ReLU$^k$ neural network of width $W$, depth $n$. That is, $h\in\Sigma_{W,n}^k$ if and only if there exists weights and biases
\begin{equation}\label{eqn:para_deep_def}
    \begin{split}
		&\bw^{(1,1)},\dots,\bw^{(1,W)}\in\RR^{d},\quad \bw^{(n+1)}\in\RR^W,\quad\bw^{(\ell,1)},\dots,\bw^{(\ell,W)}\in\RR^{W},\qquad \ell=2,\dots,n,\\		&b^{(\ell,1)},\dots,b^{(\ell,W)}\in\RR,\qquad \ell=1,\dots,n,
	\end{split}
\end{equation}
such that
		
$$h(\bx):=h^{(n+1)}(\bx)=\bw^{(n+1)}\cdot h^{(n)}(\bx),\qquad \bx\in\Omega,$$
where $h^{(n)}$ is a vector function given inductively by
\begin{equation}\label{eqn:deep_fully_layers}
	\begin{split}
		&h^{(0)}(\bx)=\bx,\\
		&\lt(h^{(\ell)}(\bx)\rt)_j=\sigma_k\lt(\bw^{(\ell,j)}\cdot h^{(\ell-1)}(\bx)+b^{(\ell,j)}\rt),\qquad j=1,\dots,W,\ \ell=1,\dots,n.
	\end{split}
\end{equation}
If we further assume all each $|\bm{b}^{(\ell,j)}|,\|\bw^{(\ell,j)}\|_\infty$, and $\|\bw^{(n+1)}\|_\infty$ is bounded by $M$, then the class is denoted as $\Sigma_{W,n}^k(M)$. For simplicity, we denote
$$\Sigma_{W,n}=\Sigma_{W,n}^1,\qquad\Sigma_{W,n}(M)=\Sigma_{W,n}^1(M).$$
\end{definition}

In terms of the number of parameters $N$, it is known that deep neural networks can approximate $f\in \mathcal{H}^r(\Omega)$ with rate $\mathcal{O}(N^{-\frac{r}{d}})$ up to logarithmic factors \cite{he2023expressivity}, which matches the classical approximation rate.

However, deep ReLU networks exhibit a superconvergence phenomenon. Yarotsky \cite{yarotsky2018optimal} showed that functions can be approximated with rate $\mathcal{O}(N^{-\frac{2r}{d}})$ using networks of constant width and depth $n$, improving upon the classical rate $\mathcal{O}(N^{-\frac{r}{d}})$. This result was further developed in \cite{shen2019deep, shen2022optimal}, where the rate is expressed as $\mathcal{O}((n^2W^2)^{-\frac{r}{d}})$, and extended to Besov spaces in \cite{siegel2023optimal}.

This apparent acceleration is notable, given that both classical methods and neural networks approximate functions from comparable smoothness classes. The following upper bound is taken from \cite{shen2022optimal}.

\begin{theorem}[Deep ReLU networks: parameters]\label{ex:fnn}
\it Let $W=d3^{d+3}$, $n\geq29+d3^{d+3}$, $M=2^{C_4n}$ for some constant $C_4=C_4(d)$, and $\Omega=[0,1]^d$, then the function
\begin{equation*}
    T_n(f)=\arg\min\limits_{g\in\Sigma_{W,n}(M)}\|f-g\|_{L^2(\Omega)}
\end{equation*}
satisfies
\begin{equation}
    \sup\limits_{f\in\BB(\mathcal{H}^r(\Omega))}\lt\|T_n(f)-f\rt\|_{L^2(\Omega)}\lesssim N^{-\frac{2r}{d}},\quad f\in\BB(\mathcal{H}^r(\Omega)),
\end{equation}
where $N$ is the total number of parameters.
\end{theorem}

Unlike the previous examples, Theorem~\ref{ex:fnn} shows that the class $\Sigma_{W,n}$ achieves an improved approximation rate in terms of the number of parameters $N$. However, as indicated by the following lemma, the associated reconstruction operator exhibits exponential sensitivity in $N$, which allows $\Sigma_{W,n}(M)$ a substantially larger metric entropy. Consequently, the precision required to represent the parameters grows rapidly, and this additional cost offsets the seeming advantage in Theorem~\ref{ex:fnn} when approximation complexity is measured in terms of bits.
\begin{lemma}\label{lem_dnn_oper_norm}
    \it Let
    \begin{equation*}
        \tilde S_N:\left\{\lt(\bw^{(\ell,j)}\rt)_{\substack{1\leq\ell\leq n+1\\1\leq j\leq W}},\lt(\bb^{(\ell)}\rt)_{\ell=1}^{n}\right\}\mapsto h^{(n+1)}
    \end{equation*}
    be the deep ReLU network as in Definition \ref{def:dnn}. Then
    \begin{equation*}
        \mathrm{Lip}(M,\tilde S_N)_{L^2(\Omega)}\lesssim (n+1)(W+1)^{n+1}M^{n+1}.
    \end{equation*}
\end{lemma}

\begin{proof}
Let the $\ell_\infty$ distance between
$$\left\{\lt(\bw^{(\ell,j)}\rt)_{\substack{1\leq\ell\leq n+1\\1\leq j\leq W}},\lt(\bb^{(\ell)}\rt)_{\ell=1}^{n}\right\}$$
and
$$\left\{\lt(\tilde\bw^{(\ell,j)}\rt)_{\substack{1\leq\ell\leq n+1\\1\leq j\leq W}},\lt(\tilde\bb^{(\ell)}\rt)_{\ell=1}^{n}\right\}$$
be bounded by $\delta$, and let $h^{(n+1)}$ and $\tilde h^{(n+1)}$ be the corresponding deep network functions. We prove by induction that, for $1\leq \ell\leq n$,
    \begin{equation}\label{eqn:hidlayer_bits_err}
		\begin{split}
			&\lt\|h^{(\ell)}\rt\|_\infty=\mal_{1\leq j\leq W}\mal_{\bx\in\Omega}\lt|\lt(h^{(\ell)}(\bx)\rt)_j\rt|\leq(W+1)^\ell M^\ell,\quad \lt\|\tilde h^{(\ell)}\rt\|_\infty\leq(W+1)^\ell M^\ell,\\
			&\lt\|h^{(\ell)}-\tilde h^{(\ell)}\rt\|_\infty=\mal_{1\leq j\leq W}\mal_{\bx\in\Omega}\lt|\lt(h^{(\ell)}(\bx)-\tilde h^{(\ell)}(\bx)\rt)_j\rt|\leq\ell(W+1)^\ell M^{\ell}\delta.
		\end{split}
	\end{equation}
	Assume that \eqref{eqn:hidlayer_bits_err} holds at the $\ell$-th layer. Then, for the $(\ell+1)$-th layer and $1\leq j\leq W$,
	\begin{equation*}
		\begin{split}
			\lt|\lt(h^{(\ell+1)}(\bx)\rt)_j\rt|&=\lt|\sigma\lt(\bw^{(\ell+1,j)}\cdot h^{(\ell)}(\bx)+b^{(\ell+1,j)}\rt)\rt|\\
            &\leq\lt\|\bw^{(\ell+1,j)}\rt\|_1\lt\|h^{(\ell)}\rt\|_\infty+\lt|b^{(\ell+1,j)}\rt|\\
			&\leq MW(W+1)^\ell M^\ell+M\\
            &\leq(W+1)^{\ell+1} M^{\ell+1}.
		\end{split}
	\end{equation*}
Moreover,
	\begin{equation*}
		\begin{split}
			&\lt|\lt(h^{(\ell+1)}(\bx)\rt)_j-\lt(\tilde h^{(\ell+1)}(\bx)\rt)_j\rt|\\
			=&\lt|\sigma\lt(\bw^{(\ell+1,j)}\cdot h^{(\ell)}(\bx)+b^{(\ell+1,j)}\rt)-\sigma\lt(\tilde \bw^{(\ell+1,j)}\cdot \tilde h^{(\ell)}(\bx)+\tilde b^{(\ell+1,j)}\rt)\rt|\\
			\leq&\lt\|\bw^{(\ell+1,j)}-\tilde \bw^{(\ell+1,j)}\rt\|_1\lt\|h^{(\ell)}\rt\|_\infty+\lt\|\bw^{(\ell+1,j)}\rt\|_1\lt\|h^{(\ell)}-\tilde h^{(\ell)}\rt\|_\infty+\lt|b^{(\ell+1,j)}-\tilde b^{(\ell+1,j)}\rt|\\
			\leq&\delta W(W+1)^\ell M^\ell+MW\ell(W+1)^\ell M^{\ell}\delta+\delta\\
			\leq&(\ell+1)(W+1)^{\ell+1} M^{\ell+1}\delta.
		\end{split}
	\end{equation*}
At the last layer, we have
\begin{equation*}
    \lt\|\bw^{(n+1)}\rt\|_\infty\leq M,\qquad
    \lt\|\bw^{(n+1)}-\tilde \bw^{(n+1)}\rt\|_\infty\leq \delta.
\end{equation*}
Together with \eqref{eqn:hidlayer_bits_err}, this gives
	\begin{equation*}
		\begin{split}
			\lt\|h^{(n+1)}-\tilde h^{(n+1)}\rt\|_\infty
            &=\lt\|\bw^{(n+1)}\cdot h^{(n)}-\tilde \bw^{(n+1)}\cdot \tilde h^{(n)}\rt\|_\infty\\
			&\leq\lt\|\bw^{(n+1)}-\tilde \bw^{(n+1)}\rt\|_1\lt\|h^{(n)}\rt\|_\infty+\lt\|\bw^{(n+1)}\rt\|_1\lt\|h^{(n)}-\tilde h^{(n)}\rt\|_\infty\\
			&\leq W(W+1)^nM^n\delta+MWn(W+1)^{n} M^{n}\delta\\
			&\leq(n+1)(W+1)^{n+1}M^{n+1}\delta.
		\end{split}
	\end{equation*}

\end{proof}

\begin{corollary}[Deep ReLU networks: bits]
    Let $\Nb\in\NN$, the operator $T_n$ in Theorem \ref{ex:fnn} induces a binary bits approximation for functions in $\BB(\mathcal{H}^{r}(\Omega))$
    \begin{equation}
        \{0,1\}^{\Nb}\mapsto T_n(f)^{[\tau]}
    \end{equation}
    such that
    $$\sup\limits_{f\in\BB(\mathcal{H}^r(\Omega))}\lt\|f- T_n(f)^{[\tau]}\rt\|_{L^2(\Omega)}\lesssim\Nb^{-\frac{2r}{3d}},$$
    and the corresponding lower bound is
    $$\inf\limits_{\psi,\phi}\sup\limits_{f\in\BB(\mathcal{H}^r(\Omega))}\lt\|\phi\circ\psi(f)-f\rt\|_{L^2(\Omega)}\gtrsim\Nb^{-\frac{r}{d}},$$
    where the infimum is taken over all $\psi:\BB(\mathcal{H}^r(\Omega))\to\{0,1\}^{\Nb}$ and $\phi:\{0,1\}^{\Nb}\to L^2(\Omega)$.
\end{corollary}
\begin{proof}

For $f\in\BB(\mathcal{H}^{r}(\Omega))$, let $T_n(f)=h^{(n+1)}$
be the minimizer in Theorem \ref{ex:fnn} with $\|T_n(f)-f\|_{L^2(\Omega)}=\mathcal{O}(N^{-\frac{2r}{d}})$, and $\lt(\bw^{(\ell,j)}(f)\rt)_{\substack{1\leq\ell\leq n+1\\1\leq j\leq W}}$, $\lt(\bb^{(\ell)}(f)\rt)_{\ell=1}^{n}$ be the parameters generating the collections $\lt\{h^{(\ell)}\rt\}_{\ell=1}^{n+1}$.

Denoting
\begin{equation*}
    \begin{split}
        &S_N:~f\mapsto\left\{\lt(\bw^{(\ell,j)}(f)\rt)_{\substack{1\leq\ell\leq n+1\\1\leq j\leq W}},\lt(\bb^{(\ell)}(f)\rt)_{\ell=1}^{n}\right\},\\
        &\tilde S_N:\left\{\lt(\bw^{(\ell,j)}\rt)_{\substack{1\leq\ell\leq n+1\\1\leq j\leq W}},\lt(\bb^{(\ell)}\rt)_{\ell=1}^{n}\right\}\mapsto h^{(n+1)},
    \end{split}
\end{equation*}
where $h^{(n+1)}$ is given by the parameters $\left\{\lt(\bw^{(\ell,j)}\rt)_{\substack{1\leq\ell\leq n+1\\1\leq j\leq W}},\lt(\bb^{(\ell)}\rt)_{\ell=1}^{n}\right\}$ as in Definition \ref{def:dnn}. By Lemma \ref{lem_dnn_oper_norm},
\begin{equation*}
    \mathrm{Lip}(M,\tilde S_N)_{L^2(\Omega)}\lesssim (n+1)(W+1)^{n+1}M^{n+1}.
\end{equation*}
By Proposition \ref{prop_para_to_bits}, taking some constant $C_5=C_5(d,r,W)$ such that
$$\tau=C_5N^2\geq 2+\left\lceil\log_2\left(M\mathrm{Lip}(M,\tilde S_N)_{L^2(\Omega)}N^{\frac{2r}{d}}\right)\right\rceil$$
and setting $T_n(f)^{[\tau]}=\tilde S_N(S_N(f)^{[\tau]})$, we have
\begin{equation*}
	\begin{split}
		\lt\|f- T_n(f)^{[\tau]}\rt\|_{L^2(\Omega)}
		\lesssim N^{-\frac{2r}{d}}\lesssim\Nb^{-\frac{2r}{3d}},
	\end{split}
\end{equation*}
where $\Nb=\tau N$ is the number of binary bits that represents $T_n(f)^{[\tau]}$.

For the lower bound, by Proposition \ref{prop:entropy_equiv} and Theorem \ref{thm:entropy_sob},
\begin{equation}
    \begin{split}
        \inf\limits_{\psi,\phi}\sup\limits_{f\in\BB(\mathcal{H}^r(\Omega))}\lt\|\phi\circ\psi(f)-f\rt\|_{L^2(\Omega)}\gtrsim\Nb^{-\frac{r}{d}}
    \end{split}
\end{equation}
where the infimum is taken over all $\psi:\BB(\mathcal{H}^r(\Omega))\to\{0,1\}^{\Nb}$ and $\phi:\{0,1\}^{\Nb}\to L^2(\Omega)$.

\end{proof}

\begin{remark}
    \rm The parameter bound $M = 2^{C_4 n}$ is a sufficient condition and is not necessarily sharp. A more refined construction of deep neural network architectures than that in \cite{shen2022optimal} may lead to improved bit-based approximation rates, potentially surpassing the bound $\Nb^{-\frac{2r}{3d}}$.
\end{remark}

In terms of the number of parameters, deep ReLU networks achieve the superconvergence rate $\mathcal{O}(N^{-\frac{2r}{d}})$ for Sobolev functions. However, Lemma~\ref{lem_dnn_oper_norm} shows that the reconstruction map has much stronger sensitivity to perturbations of the parameters, so maintaining the same approximation accuracy requires substantially higher precision. Consequently, the bit-based rate in Section~\ref{ex:fnn_bits} does not reflect the same improvement as the parameter-based rate.

This also clarifies the role of finite precision in practical computation. For the methods considered in Sections~\ref{ex:polynomial_bits}--\ref{ex:barron_bits}, the required number of bits per parameter grows only logarithmically with the number of parameters. In contrast, the sufficient precision for the deep network construction above grows polynomially as the stability bound in Lemma~\ref{lem_dnn_oper_norm} is much larger. Thus, two methods with similar parameter counts may have very different bit-complexities. This distinction is especially important in practice, where finite-precision storage and computation are unavoidable.

We next explain how metric entropy provides lower bounds for approximation and gives a benchmark for evaluating parameter-based methods.

\section{Metric entropy and approximation lower bound}\label{sec:entropy_appr}	
In Sections \ref{sec:classical}--\ref{sec:appr_nn}, we studied approximation methods based on parameter representations and analyzed their performance in terms of binary bits. We now take a different perspective. Since approximation rates are typically expressed in terms of parameters, it is natural to ask whether a given parameter-based approach is efficient.

Recall that Proposition \ref{prop:entropy_equiv} establishes the equivalence between the entropy number $\epsilon_m(K)_X$ and the minimal approximation error achievable with $m$ bits. Therefore, metric entropy provides a fundamental benchmark for evaluating parameter-based approximation methods.






	\begin{lemma}\label{lem:entropy_lower_gener}
		\it Let $\mathcal{X}$ be a metric space, $\mathcal{A}$ and $\mathcal{B}$ be two subsets of $\mathcal{X}$, then
		$$\sup\limits_{f\in\mathcal{B}}\inf\limits_{g\in\mathcal{A}}\|f-g\|_{\mathcal{X}}\geq\epsilon_m(\mathcal{B})_{\mathcal{X}}-\epsilon_m(\mathcal{A})_{\mathcal{X}},\quad\forall m\in\mathbb{N}.$$
	\end{lemma}
	\begin{proof}
		By Proposition \ref{prop:entropy_equiv}, for any $\eta>0$, there exist $\phi:\{0,1\}^m\to\mathcal{X}$, s.t.
		$$\sup\limits_{g\in\mathcal{A}}\inf\limits_{y\in\{0,1\}^m}\|g-\phi(y)\|_{\mathcal{X}}\leq\epsilon_m(\mathcal{A})+\eta.$$
        Then for any $f\in\mathcal{B},~g\in\mathcal{A}$,
        \begin{equation}
            \inf\limits_{y\in\{0,1\}^m}\|f-\phi(y)\|_{\mathcal{X}}\leq\inf\limits_{y\in\{0,1\}^m}\|g-\phi(y)\|_{\mathcal{X}}+\|f-g\|_{\mathcal{X}}
				\leq\epsilon_m(\mathcal{A})+\eta+\|f-g\|_{\mathcal{X}}.
        \end{equation}
        Taking infimum of $g\in\mathcal{A}$ and supremum of $f\in\mathcal{B}$ yields
		\begin{equation*}
				\sup\limits_{f\in\mathcal{B}}\inf\limits_{y\in\{0,1\}^m}\|f-\phi(y)\|_{\mathcal{X}}\leq\epsilon_m(\mathcal{A})+\eta+\sup\limits_{f\in\mathcal{B}}\inf\limits_{g\in\mathcal{A}}\|f-g\|_{\mathcal{X}}
		\end{equation*}
		Using Proposition \ref{prop:entropy_equiv} again and taking $\eta\to0$,
        \begin{equation}
            \epsilon_m(\mathcal{B})_{\mathcal{X}}\leq\epsilon_m(\mathcal{A})+\sup\limits_{f\in\mathcal{B}}\inf\limits_{g\in\mathcal{A}}\|f-g\|_{\mathcal{X}}
        \end{equation}
	\end{proof}

    In the rest of this subsection, we provide some examples to illustrate how the lower bound can be derived in concrete approximation problems.

    \subsection{Examples of lower bounds in approximation}
    Given the metric entropies, we are ready to apply Lemma \ref{lem:entropy_lower_gener} to prove lower bounds of the approximation rates.

	\subsubsection{Approximating functions in $\mathcal{H}^r(\Omega)$ by polynomials}\label{subsec:appr_rate_polynomial}    
    Let
    $$\mathcal{A}_{n}:=\Bigg\{\sul_{\substack{\|\bk\|_\infty<n\\ \bk\in\NN^d}}a_{\bk,n}P_\bk:~|a_{\bk,n}|\leq 1\Bigg\}$$
    be the approximant class in Theorem \ref{ex:polynomial}, we construct a finite subset by quantizing the coefficients. For $m>n^d$,
    by taking $L=\lt\lfloor2^{\frac{m}{n^d}}-1\rt\rfloor$, the class
    $$\tilde{\mathcal{A}}_{n}(L):=\Bigg\{\sul_{\|\bk\|_\infty<n}\theta_\bk P_\bk:~\theta_\bk\in E_L=\lt\{\frac{-2L+1}{2L},\frac{-2L+3}{2L},\dots,\frac{2L-1}{2L}\rt\}\Bigg\}$$
    with $\#\tilde{\mathcal{A}}_{n}\leq2^m$ yields an $A_1n^{\frac{d}{2}}2^{-\frac{m}{n^d}}$-cover of $\mathcal{A}_{n}$ with a constant $A_1$ independent of $n,m$. As a result,
    $$\epsilon_m(\mathcal{A}_{n})_{L^2(\Omega)}\leq A_1n^{d\over2}2^{-\frac{m}{n^{d}}}.$$
	Together with Theorem \ref{thm:entropy_sob} and Lemma \ref{lem:entropy_lower_gener}, we conclude
	\begin{equation*}
		\sup\limits_{f\in \BB\lrt{\mathcal{H}^{r}(\Omega)}}\inf\limits_{g\in \mathcal{A}_{n}}\lt\|f-g\rt\|_{L^2(\Omega)}\geq A_2m^{-\frac{r}{d}}-A_1n^{d\over2}2^{-\frac{m}{n^{d}}},\qquad m\geq n^d.
	\end{equation*}
	Taking $m=(1+\frac{r}{d})N\log_2N$ yields
	\begin{equation}
		\sup\limits_{f\in \BB\lrt{\mathcal{H}^{r}(\Omega)}}\inf\limits_{g\in \mathcal{A}_{n}}\lt\|f-g\rt\|_{L^2(\Omega)}\gtrsim (N\log N)^{-\frac{r}{d}}.
	\end{equation}
where $N=n^d$ is the number of parameters in $\mathcal{A}_{n}$. Compared with this lower bound, the rate
$$\|f-T_n(f)\|_{L^2(\Omega)}\leq\|f\|_{\mathcal{H}^{r}(\Omega)} N^{-\frac{r}{d}}$$
in Theorem \ref{ex:polynomial} is suboptimal among all the encoding approaches.

\subsubsection{Approximating functions in Barron spaces by shallow networks}\label{subsec:appr_rate_variation}
For $m>n(d+2)$,
    by taking $L=\lt\lfloor2^{\frac{m}{n(d+2)}}-1\rt\rfloor$, the class
    \begin{equation*}
			\tilde{\mathcal{A}}_{n}:=\lt\{\sum\limits_{j=1}^n\tilde a^{(j)}\sigma_k(\tilde \bw^{(j)}\cdot\bx+\tilde b^{(j)}):~\tilde a_j\in ME_L,\tilde b_j\in E_L,~\tilde w^{(j)}_i\in E_L,~\forall i,j\rt\},
		\end{equation*}
    with $\#\tilde{\mathcal{A}}_{n}\leq2^m$ yields a $A_32^{-\frac{m}{n(d+2)}}$-cover of $\Sigma_n^k(M)$ with a constant $A_3$ independent of $n,m$. As a result,
    \begin{equation}\label{eqn_entropy_snn}
        \epsilon_m(\Sigma_n^k(M))_{L^2(\Omega)}\leq A_3Mn2^{-\frac{m}{n(d+2)}}.
    \end{equation}
	Together with Theorem \ref{thm:entropy_variation} and Lemma \ref{lem:entropy_lower_gener}, we conclude
    \begin{equation*}
	\sup\limits_{f\in \BB\lrt{\mathcal{B}^k(\Omega)}}\inf\limits_{g\in \Sigma_n^k(M)}\lt\|f-g\rt\|_{L^2(\Omega)}\geq A_4m^{-\frac{1}{2}-\frac{2k+1}{2d}}-A_3Mn2^{-\frac{m}{n(d+2)}},\qquad m>(d+2)n
\end{equation*}
By taking $m=(d+2)(2+\frac{2k+1}{2d})n\log_2n$
\begin{equation*}
	\sup\limits_{f\in \BB\lrt{\mathcal{B}^k(\Omega)}}\inf\limits_{g\in \Sigma_{n}^k(M)}\lt\|f-g\rt\|_{L^2(\Omega)}\gtrsim \lt(N\log_2N\rt)^{-\frac{1}{2}-\frac{2k+1}{2d}},
\end{equation*}
where $N=(d+2)n$ is the number of parameters.

Compared with this lower bound, the rate
$$\|T_n(f)-f\|_{L^2(\Omega)}\lesssim\|f\|_{\mathcal{B}^k(\Omega)}N^{-\frac{1}{2}-\frac{2k+1}{2d}},$$
in \eqref{eqn:variation_appr} is suboptimal.

\subsubsection{Approximating functions in spectral Barron spaces by shallow networks}\label{subsec:appr_rate_barron}
By substituting \eqref{eqn_entropy_snn} and Theorem \ref{thm:entropy_variation} into Lemma \ref{lem:entropy_lower_gener}, for $s=(d+1)(k+1)$, we conclude
\begin{equation*}
	\sup\limits_{f\in \BB\lrt{\mathcal{V}^s(\Omega)}}\inf\limits_{g\in \Sigma_n^k(M)}\lt\|f-g\rt\|_{L^2(\Omega)}\geq A_5m^{-\frac{1}{2}-\frac{s}{d}}-A_3Mn2^{-\frac{m}{n(d+2)}},\qquad m>(d+2)n,
\end{equation*}
which implies
\begin{equation}\label{eqn:snn_appr_unif_lower}
	\sup\limits_{f\in \BB\lrt{\mathcal{V}^s(\Omega)}}\inf\limits_{g\in \Sigma_{n}^k(M)}\lt\|f-g\rt\|_{L^2(\Omega)}\gtrsim \lt(N\log_2N\rt)^{-(k+1)\frac{d+1}{d}-\frac{1}{2}},
\end{equation}
where $N=(d+2)n$ is the number of parameters.

However, the approximation rate we applied in Theorem \ref{ex:barron} is the rate
\begin{equation}\label{eqn:snn_appr_unif_upper}
    \|T_n(f)-f\|_{L^2(\Omega)}\lesssim\|f\|_{\mathcal{V}^s(\Omega)}N^{-(k+1)}\log N,
\end{equation}
given in \eqref{eqn:barron_appr}. We can see there is a gap between \eqref{eqn:snn_appr_unif_lower} and \eqref{eqn:snn_appr_unif_upper}. It is possible that the rate in Theorem \ref{ex:barron} is not optimal. We might need to find other approaches to improve the rate of approximation for functions from $\mathcal{V}^s(\Omega)$.

	\subsection{Suboptimality in approximation and binary representation}\label{subsec:rela_appr_bin}
	When we are representing elements in $\mathcal{B}$ by binary representation of elements in a set $\mathcal{A}$ with smaller metric entropy, Lemma \ref{lem:entropy_lower_gener} gives a necessary condition that we can obtain the suboptimal binary bits approximation.

\begin{enumerate}

\item \textbf{Approximating functions in Sobolev spaces by orthogonal polynomials:}  When representing functions in $\BB\lrt{\mathcal{H}^{r}(\Omega)}$ through the binary approximation of partial Fourier sums in
$$\mathcal{A}_{n}:=\lt\{\sul_{\|\bk\|_\infty< n}a_{\bk} P_\bk:~|a_{\bk}|\leq1\rt\},$$
as we discussed in Section \ref{subsec:classes_entropy}, the binary approximation through $\mathcal{A}_n$ is suboptimal. Necessarily, the approximation rate $\sup\limits_{f\in \BB\lrt{\mathcal{H}^{r}(\Omega)}}\inf\limits_{g\in \mathcal{A}_{n}}\lt\|f-g\rt\|_{L^2(\TT^d)}$ is suboptimal compared with the rate given by Lemma \ref{lem:entropy_lower_gener}.

\item \textbf{Approximating functions in Barron spaces by shallow ReLU$^k$ networks:}  When representing functions in $\BB\lrt{\mathcal{B}^k(\Omega)}$ through the binary approximation of shallow ReLU$^k$ network functions in $\Sigma_n^k(M)$,
as we discussed in Section \ref{subsec:classes_entropy}, the binary approximation through $\Sigma_n^k(M)$ is suboptimal. Necessarily, the approximation rate $\sup\limits_{f\in \BB\lrt{\mathcal{B}^k(\Omega)}}\inf\limits_{g\in \Sigma_n^k(M)}\lt\|f-g\rt\|_{L^2(\Omega)}$ is suboptimal compared with the rate given by Lemma \ref{lem:entropy_lower_gener}. This is shown in Section \ref{subsec:appr_rate_variation} directly.

\item \textbf{Approximating functions in spectral Barron spaces by shallow ReLU$^k$ networks:}  However, when representing functions in $\BB\lrt{\mathcal{V}^s(\Omega)}$ through the binary approximation of shallow ReLU$^k$ network functions in $\Sigma_n^k(M)$,
we see the upper bound
$$\|T_n(f)-f\|_{L^2(\Omega)}\lesssim\|f\|_{\mathcal{V}^s(\Omega)}N^{-(k+1)}\log N$$
in Section \ref{subsec:appr_rate_barron} has a gap with the lower bound \eqref{eqn:snn_appr_unif_lower} given by Lemma \ref{lem:entropy_lower_gener}. As a result, the binary approximation of $\BB\lrt{\mathcal{V}^s(\Omega)}$ through that of $\Sigma_n^k(M)$ cannot be optimal as well.

\end{enumerate}

\section{Concluding remarks}\label{sec:conclusion}

In this paper, we revisit several classical and neural network approximation examples from the viewpoint of binary bits, leading to two main conclusions concerning dimension-independent rates and deep-network superconvergence.

As shown in Sections~\ref{sec:classical}--\ref{sec:appr_nn}, although dimension-independent rates of neural networks appear stronger than classical approximation rates, the bit-complexity perspective gives a different interpretation. The metric entropy estimates show that the target function classes have relatively small bit-complexity, comparable to that of classical function classes with higher regularity. Thus, the dimension-independent behavior should be understood as a consequence of the small complexity of the target classes rather than as an intrinsic advantage of the network architecture.

Deep network superconvergence is also a parameter-based phenomenon. As shown in Section~\ref{ex:fnn_bits}, in the deep network architecture, the required bits for representing the parameters can increase substantially, so the faster parameter-based rate for deep networks does not necessarily translate into a faster bit-complexity rate. In other words, the seeming advantage in parameter count can be offset by the bit cost of stable finite-precision representation.

Overall, the fundamental limitation in approximation is not dimension. From the perspective developed in this paper, approximation rates should be understood in terms of the bit-complexity of both the target function classes and the corresponding classes of approximants.

\noindent {\bf Acknowledgments.} The authors are supported by the KAUST Baseline Research Fund.


\begin{thebibliography}{99}

\bibitem{abramowitz1964handbook} M.~Abramowitz and I.~A. Stegun, \emph{Handbook of Mathematical Functions: With Formulas, Graphs, and Mathematical Tables}, vol.~55, US Government Printing Office, 1964.

\bibitem{arcangeli2007extension} R.~Arcang{\'e}li, M.~C. L{\'o}pez de Silanes, and J.~J. Torrens, An extension of a bound for functions in Sobolev spaces, with applications to $(m,s)$-spline interpolation and smoothing, \emph{Numerische Mathematik}, \textbf{107}(2) (2007), 181--211.

\bibitem{arnold1958representation} V.~I. Arnold, The representation of functions of several variables, \emph{Mat. Prosvesc.}, \textbf{3} (1958), 41--61.

\bibitem{barron1991complexity} A.~R. Barron, Complexity regularization with application to artificial neural networks, in \emph{Nonparametric Functional Estimation and Related Topics}, Springer, 1991, 561--576.

\bibitem{Barron2008} A.~R. Barron, A.~Cohen, W.~Dahmen, and R.~A. DeVore, Approximation and learning by greedy algorithms, \emph{Annals of Statistics}, \textbf{36}(1) (2008), 64--94.

\bibitem{bartlett1996sample} P.~Bartlett, The sample complexity of pattern classification with neural networks: the size of the weights is more important than the size of the network, \emph{IEEE Transactions on Information Theory}, \textbf{44}(2) (1998), 525--536.

\bibitem{bartlett2002rademacher} P.~L. Bartlett and S.~Mendelson, Rademacher and Gaussian complexities: risk bounds and structural results, \emph{Journal of Machine Learning Research}, \textbf{3} (2002), 463--482.

\bibitem{bathe2006finite} K.-J. Bathe, \emph{Finite Element Procedures}, Klaus-Jurgen Bathe, 2006.

\bibitem{bengio2009learning} Y.~Bengio, Learning deep architectures for AI, \emph{Foundations and Trends in Machine Learning}, \textbf{2}(1) (2009), 1--127.

\bibitem{bertoluzza2012primer} S.~Bertoluzza, R.~H. Nochetto, A.~Quarteroni, K.~G. Siebert, and A.~Veeser, Primer of adaptive finite element methods, in \emph{Multiscale and Adaptivity: Modeling, Numerics and Applications}, CIME Summer School, Cetraro, Italy, 2009, Springer, 2012, 125--225.

\bibitem{brenner2008mathematical} S.~C. Brenner, \emph{The Mathematical Theory of Finite Element Methods}, Springer, 2008.

\bibitem{bungartz2004sparse} H.-J. Bungartz and M.~Griebel, Sparse grids, \emph{Acta Numerica}, \textbf{13} (2004), 147--269.

\bibitem{chen2004support} D.-R. Chen, Q.~Wu, Y.~Ying, and D.-X. Zhou, Support vector machine soft margin classifiers: error analysis, \emph{Journal of Machine Learning Research}, \textbf{5} (2004), 1143--1175.

\bibitem{ciarlet2002finite} P.~G. Ciarlet, \emph{The Finite Element Method for Elliptic Problems}, SIAM, 2002.

\bibitem{cucker2002mathematical} F.~Cucker and S.~Smale, On the mathematical foundations of learning, \emph{Bulletin of the American Mathematical Society}, \textbf{39}(1) (2002), 1--49.

\bibitem{cucker2007learning} F.~Cucker and D.~X. Zhou, \emph{Learning Theory: An Approximation Theory Viewpoint}, vol.~24, Cambridge University Press, 2007.

\bibitem{davis1975interpolation} P.~J. Davis, \emph{Interpolation and Approximation}, Dover Publications, 1975.

\bibitem{devore1989optimal} R.~A. DeVore, R.~Howard, and C.~Micchelli, Optimal nonlinear approximation, \emph{Manuscripta Mathematica}, \textbf{63} (1989), 469--478.

\bibitem{devore1993constructive} R.~A. DeVore and G.~G. Lorentz, \emph{Constructive Approximation}, vol.~303, Springer Science \& Business Media, 1993.

\bibitem{devore1996some} R.~A. DeVore and V.~N. Temlyakov, Some remarks on greedy algorithms, \emph{Advances in Computational Mathematics}, \textbf{5}(1) (1996), 173--187.

\bibitem{eldan2016power} R.~Eldan and O.~Shamir, The power of depth for feedforward neural networks, in \emph{Conference on Learning Theory}, PMLR, 2016, 907--940.

\bibitem{feng2021generalization} H.~Feng, S.~Huang, and D.-X. Zhou, Generalization analysis of CNNs for classification on spheres, \emph{IEEE Transactions on Neural Networks and Learning Systems}, 2021.

\bibitem{harvey2017nearly} N.~Harvey, C.~Liaw, and A.~Mehrabian, Nearly-tight VC-dimension bounds for piecewise linear neural networks, in \emph{Conference on Learning Theory}, PMLR, 2017, 1064--1068.

\bibitem{he2023expressivity} J.~He, T.~Mao, and J.~Xu, Expressivity and approximation properties of deep neural networks with ReLU$^k$ activation, \emph{arXiv preprint arXiv:2312.16483}, 2023.

\bibitem{jiao2021deep} Y.~Jiao, Y.~Lai, X.~Lu, and Z.~Yang, Deep neural networks with ReLU-sine-exponential activations break curse of dimensionality on H{\"o}lder class, \emph{arXiv preprint arXiv:2103.00542}, 2021.

\bibitem{jin2015finite} J.-M. Jin, \emph{The Finite Element Method in Electromagnetics}, John Wiley \& Sons, 2015.

\bibitem{Jones1992} L.~K. Jones, A simple lemma on greedy approximation in Hilbert space and convergence rates for projection pursuit regression and neural network training, \emph{The Annals of Statistics}, \textbf{20}(1) (1992), 608--613.

\bibitem{karvonen2025superconvergence} T.~Karvonen, G.~Santin, and T.~Wenzel, General superconvergence for kernel-based approximation, \emph{arXiv preprint arXiv:2505.11435}, 2025.

\bibitem{klusowski2018approximation} J.~M. Klusowski and A.~R. Barron, Approximation by combinations of ReLU and squared ReLU ridge functions with $\ell^1$ and $\ell^0$ controls, \emph{IEEE Transactions on Information Theory}, \textbf{64}(12) (2018), 7649--7656.

\bibitem{kolmogoro1956representation} A.~N. Kolmogorov, On the representation of continuous functions of several variables as superpositions of functions of smaller number of variables, \emph{Soviet Math. Dokl.}, \textbf{108} (1956), 179--182.

\bibitem{kolmogorov1958definition} A.~N. Kolmogorov and V.~A. Uspenskii, On the definition of an algorithm, \emph{Uspekhi Matematicheskikh Nauk}, \textbf{13}(4) (1958), 3--28.

\bibitem{koltchinskii2001rademacher} V.~Koltchinskii, Rademacher penalties and structural risk minimization, \emph{IEEE Transactions on Information Theory}, \textbf{47}(5) (2001), 1902--1914.

\bibitem{kurkova1} V.~K{\r u}rkov{\'a} and M.~Sanguineti, Bounds on rates of variable basis and neural network approximation, \emph{IEEE Transactions on Information Theory}, \textbf{47}(6) (2001), 2659--2665.

\bibitem{kurkova2} V.~K{\r u}rkov{\'a} and M.~Sanguineti, Comparison of worst case errors in linear and neural network approximation, \emph{IEEE Transactions on Information Theory}, \textbf{48}(1) (2002), 264--275.

\bibitem{kulkarni1989metric} S.~R. Kulkarni et al., On metric entropy, Vapnik--Chervonenkis dimension, and learnability for a class of distributions, 1989.

\bibitem{lewicki2004approximation} G.~Lewicki and G.~Marino, Approximation of functions of finite variation by superpositions of a sigmoidal function, \emph{Applied Mathematics Letters}, \textbf{17}(10) (2004), 1147--1152.

\bibitem{lin2014lower} Q.~Lin, H.~Xie, and J.~Xu, Lower bounds of the discretization error for piecewise polynomials, \emph{Mathematics of Computation}, \textbf{83}(285) (2014), 1--13.

\bibitem{liu2025achieving} X.~Liu, T.~Mao, and J.~Xu, Integral representations of Sobolev spaces via ReLU$^k$ activation function and optimal error estimates for linearized networks, \emph{arXiv preprint arXiv:2505.00351}, 2025.

\bibitem{lorentz1996constructive} G.~G. Lorentz, M.~von Golitschek, and Y.~Makovoz, \emph{Constructive Approximation: Advanced Problems}, vol.~304, Springer, 1996.

\bibitem{ma2022uniform} L.~Ma, J.~W. Siegel, and J.~Xu, Uniform approximation rates and metric entropy of shallow neural networks, \emph{Research in the Mathematical Sciences}, \textbf{9}(3) (2022), 46.

\bibitem{maiorov1999lower} V.~Maiorov and A.~Pinkus, Lower bounds for approximation by MLP neural networks, \emph{Neurocomputing}, \textbf{25}(1--3) (1999), 81--91.

\bibitem{mao2022approximation} T.~Mao and D.-X. Zhou, Approximation of functions from Korobov spaces by deep convolutional neural networks, \emph{Advances in Computational Mathematics}, \textbf{48}(6) (2022), 84.

\bibitem{mhaskar2004kernel} H.~N. Mhaskar, On the representation of smooth functions by radial basis functions, \emph{Journal of Approximation Theory}, \textbf{127}(1) (2004), 1--16.

\bibitem{mhaskar2010eignets} H.~N. Mhaskar, Eignets for function approximation on manifolds, \emph{Applied and Computational Harmonic Analysis}, \textbf{29}(1) (2010), 63--87.

\bibitem{mhaskar2020kernel} H.~N. Mhaskar, Kernel-based analysis of massive data, \emph{Frontiers in Applied Mathematics and Statistics}, \textbf{6} (2020), 30.

\bibitem{mhaskar1994dimension} H.~N. Mhaskar and C.~A. Micchelli, Dimension-independent bounds on the degree of approximation by neural networks, \emph{IBM Journal of Research and Development}, \textbf{38}(3) (1994), 277--284.

\bibitem{mhaskar1999zonal} H.~N. Mhaskar, F.~J. Narcowich, and J.~D. Ward, Approximation properties of zonal function networks using scattered data on the sphere, \emph{Advances in Computational Mathematics}, \textbf{11} (1999), 121--137.

\bibitem{moaveni2011finite} S.~Moaveni, \emph{Finite Element Analysis: Theory and Application with ANSYS}, 3rd ed., Pearson Education India, 2011.

\bibitem{mohri2018foundations} M.~Mohri, A.~Rostamizadeh, and A.~Talwalkar, \emph{Foundations of Machine Learning}, 2nd ed., MIT Press, 2018.

\bibitem{montanelli2019new} H.~Montanelli and Q.~Du, New error bounds for deep ReLU networks using sparse grids, \emph{SIAM Journal on Mathematics of Data Science}, \textbf{1}(1) (2019), 78--92.

\bibitem{narcowich2005sobolev} F.~J. Narcowich, J.~D. Ward, and H.~Wendland, Sobolev bounds on functions with scattered zeros, with applications to radial basis function surface fitting, \emph{Mathematics of Computation}, \textbf{74}(250) (2005), 743--763.

\bibitem{narcowich2006sobolev} F.~J. Narcowich, J.~D. Ward, and H.~Wendland, Sobolev error estimates and a Bernstein inequality for scattered data interpolation via radial basis functions, \emph{Constructive Approximation}, \textbf{24}(2) (2006), 175--186.

\bibitem{petrushev1998approximation} P.~P. Petrushev, Approximation by ridge functions and neural networks, \emph{SIAM Journal on Mathematical Analysis}, \textbf{30}(1) (1998), 155--189.

\bibitem{poggio2017and} T.~Poggio, H.~Mhaskar, L.~Rosasco, B.~Miranda, and Q.~Liao, Why and when can deep-but not shallow-networks avoid the curse of dimensionality: a review, \emph{International Journal of Automation and Computing}, \textbf{14}(5) (2017), 503--519.

\bibitem{reddy1993introduction} J.~Reddy, \emph{An Introduction to the Finite Element Method}, 1993.

\bibitem{schaback1999improved} R.~Schaback, Improved error bounds for scattered data interpolation by radial basis functions, \emph{Mathematics of Computation}, \textbf{68}(225) (1999), 201--216.

\bibitem{schaback2018superconvergence} R.~Schaback, Superconvergence of kernel-based interpolation, \emph{Journal of Approximation Theory}, \textbf{235} (2018), 1--19.

\bibitem{shen2019deep} Z.~Shen, H.~Yang, and S.~Zhang, Deep network approximation characterized by number of neurons, \emph{arXiv preprint arXiv:1906.05497}, 2019.

\bibitem{shen2020deep} Z.~Shen, H.~Yang, and S.~Zhang, Deep network with approximation error being reciprocal of width to power of square root of depth, \emph{arXiv preprint arXiv:2006.12231}, 2020.

\bibitem{shen2021neural} Z.~Shen, H.~Yang, and S.~Zhang, Neural network approximation: three hidden layers are enough, \emph{Neural Networks}, \textbf{141} (2021), 160--173.

\bibitem{shen2022optimal} Z.~Shen, H.~Yang, and S.~Zhang, Optimal approximation rate of ReLU networks in terms of width and depth, \emph{Journal de Math{\'e}matiques Pures et Appliqu{\'e}es}, \textbf{157} (2022), 101--135.

\bibitem{siegel2023optimalzonoids} J.~W. Siegel, Optimal approximation of zonoids and uniform approximation by shallow neural networks, \emph{arXiv preprint arXiv:2307.15285}, 2023.

\bibitem{siegel2023optimal} J.~W. Siegel, Optimal approximation rates for deep ReLU neural networks on Sobolev and Besov spaces, \emph{Journal of Machine Learning Research}, \textbf{24}(357) (2023), 1--52.

\bibitem{siegel2020approximation} J.~W. Siegel and J.~Xu, Approximation rates for neural networks with general activation functions, \emph{Neural Networks}, \textbf{128} (2020), 313--321.

\bibitem{siegel2022high} J.~W. Siegel and J.~Xu, High-order approximation rates for shallow neural networks with cosine and ReLU$^k$ activation functions, \emph{Applied and Computational Harmonic Analysis}, \textbf{58} (2022), 1--26.

\bibitem{siegel2022sharp} J.~W. Siegel and J.~Xu, Sharp bounds on the approximation rates, metric entropy, and $n$-widths of shallow neural networks, \emph{Foundations of Computational Mathematics}, 2022, 1--57.

\bibitem{siegel2023characterization} J.~W. Siegel and J.~Xu, Characterization of the variation spaces corresponding to shallow neural networks, \emph{Constructive Approximation}, \textbf{57}(3) (2023), 1109--1132.

\bibitem{sloan2025doubling} I.~Sloan and V.~Kaarnioja, Doubling the rate: improved error bounds for orthogonal projection with application to interpolation, \emph{BIT Numerical Mathematics}, 2025.

\bibitem{stoer1980introduction} J.~Stoer, R.~Bulirsch, R.~Bartels, W.~Gautschi, and C.~Witzgall, \emph{Introduction to Numerical Analysis}, Springer, 1980.

\bibitem{suli2003introduction} E.~S{\"u}li and D.~F. Mayers, \emph{An Introduction to Numerical Analysis}, Cambridge University Press, 2003.

\bibitem{szego1939orthogonal} G.~Szeg{\H{o}}, \emph{Orthogonal Polynomials}, vol.~23, American Mathematical Society, 1939.

\bibitem{temlyakov2008greedy} V.~N. Temlyakov, Greedy approximation, \emph{Acta Numerica}, \textbf{17} (2008), 235--409.

\bibitem{timan2014theory} A.~F. Timan, \emph{Theory of Approximation of Functions of a Real Variable}, Elsevier, 2014.

\bibitem{vapnik2015uniform} V.~N. Vapnik and A.~Y. Chervonenkis, On the uniform convergence of relative frequencies of events to their probabilities, in \emph{Measures of Complexity: Festschrift for Alexey Chervonenkis}, Springer, 2015, 11--30.

\bibitem{wendland2005approximate} H.~Wendland and C.~Rieger, Approximate interpolation with applications to selecting smoothing parameters, \emph{Numerische Mathematik}, \textbf{101}(4) (2005), 729--748.

\bibitem{wenzel2026sharp} T.~Wenzel, Sharp inverse statements for kernel approximation: superconvergence and saturation, \emph{arXiv preprint arXiv:2601.01808}, 2026.

\bibitem{xu1992iterative} J.~Xu, Iterative methods by space decomposition and subspace correction, \emph{SIAM Review}, \textbf{34}(4) (1992), 581--613.

\bibitem{xu2020finite} J.~Xu, The finite neuron method and convergence analysis, \emph{arXiv preprint arXiv:2010.01458}, 2020.

\bibitem{yang2023nearly} Y.~Yang, Y.~Wu, H.~Yang, and Y.~Xiang, Nearly optimal approximation rates for deep super ReLU networks on Sobolev spaces, \emph{arXiv preprint arXiv:2310.10766}, 2023.

\bibitem{yang2023optimal} Y.~Yang and D.-X. Zhou, Optimal rates of approximation by shallow ReLU$^k$ neural networks and applications to nonparametric regression, \emph{arXiv preprint arXiv:2304.01561}, 2023.

\bibitem{yarotsky2018optimal} D.~Yarotsky, Optimal approximation of continuous functions by very deep ReLU networks, in \emph{Conference on Learning Theory}, PMLR, 2018, 639--649.

\bibitem{yarotsky2020phase} D.~Yarotsky and A.~Zhevnerchuk, The phase diagram of approximation rates for deep neural networks, \emph{Advances in Neural Information Processing Systems}, \textbf{33} (2020), 13005--13015.

\bibitem{zhang2022deep} S.~Zhang, Z.~Shen, and H.~Yang, Deep network approximation: achieving arbitrary accuracy with fixed number of neurons, \emph{Journal of Machine Learning Research}, \textbf{23}(276) (2022), 1--60.

\end{thebibliography}
\end{document}